\documentclass[twoside,11pt]{article}

\usepackage{amsmath}
\usepackage{amssymb}
\usepackage{amsthm}
\usepackage{indentfirst}
\usepackage[hidelinks]{hyperref}
\usepackage{paralist}
\usepackage{geometry}
\usepackage{setspace,lipsum}
\usepackage[title]{appendix}
\usepackage{color}
\usepackage{multimedia}
\usepackage{multirow}
\usepackage{bbm}
\usepackage{graphics}
\usepackage{graphicx}
\usepackage{booktabs}
\usepackage{romannum}
\usepackage{grffile}
\usepackage{float}
\usepackage{url}
\usepackage{mathabx}
\usepackage{supertabular}
\usepackage{rotating}
\usepackage{pdflscape}
\usepackage{verbatim}
\usepackage{listings}
\usepackage{xcolor}
\usepackage{multicol}
\usepackage[shortlabels]{enumitem} 
\usepackage{caption,subcaption}
\usepackage{epstopdf}
\usepackage{algorithm}
\usepackage{algpseudocode}
\usepackage{tikz} 
\usepackage[round]{natbib}
\usepackage{authblk}
\usepackage{etoolbox}
\graphicspath{{figure/}}
\newcommand{\bm}[1]{\boldsymbol{#1}}
\newcommand{\dd}{\mathrm{d}}

\newcommand{\E}{\mathbb{E}}

\newcommand{\g}{\mathcal{G}}

\newtheorem{theorem}{Theorem}[section]
\newtheorem{lemma}[theorem]{Lemma}
\newtheorem{remark}[theorem]{Remark}

\algblock{Input}{EndInput}
\algnotext{EndInput}
\algblock{Output}{EndOutput}
\algnotext{EndOutput}
\newcommand{\Desc}[2]{\State \makebox[5em][l]{#1}#2}

\renewcommand{\theequation}{\thesection.\arabic{equation}}
\renewcommand{\thetheorem}{\thesection.\arabic{theorem}}
\renewcommand{\thetheorem}{\thesection.\arabic{theorem}}
\renewcommand{\thelemma}{\thesection.\arabic{theorem}}  
\renewcommand{\theremark}{\thesection.\arabic{theorem}} 
\renewcommand{\thealgorithm}{\thesection.\arabic{algorithm}}

\usetikzlibrary{fit,positioning}

\newcommand{\Mj}{(\sum_{j=1}^{m} D_j D_j^\top)^{-1}}
\newcommand{\Vj}{(\sum_{j=1}^{m} C_j D_j)}
\newcommand{\Sj}{\sum_{j=1}^{m}}

\begin{document}

\pagenumbering{arabic}
\providecommand{\keywords}[1]
{
\small	
\textbf{\textit{Keywords:}} #1
}

\title{Sublinear Regret for a Class of Continuous-Time Linear--Quadratic Reinforcement Learning Problems}

\author{Yilie Huang \thanks{Department of Industrial Engineering and Operations Research, Columbia University, New York, NY 10027, USA. Email: yh2971@columbia.edu.}  ~ ~ ~  Yanwei Jia\thanks{Department of Systems Engineering and Engineering Management,  The Chinese University of Hong Kong, Hong Kong. Email: yanweijia@cuhk.edu.hk.} ~ ~ ~ Xun Yu Zhou\thanks{Department of Industrial Engineering and Operations Research \& Data Science Institute, Columbia University, New York, NY 10027, USA. Email: xz2574@columbia.edu.}}

\date{}
\maketitle

\let\thefootnote\relax
\footnotetext{Accepted for publication in \textit{SIAM Journal on Control and Optimization (2025)}.}

\begin{abstract}
\singlespacing
We study reinforcement learning (RL) for a class of continuous-time linear--quadratic (LQ) control problems for diffusions, where states are scalar-valued and running control rewards are absent but volatilities of the state processes depend on both state and control variables. We apply a model-free approach that relies neither on knowledge of  model parameters nor on their estimations, and devise an RL algorithm to learn the optimal policy parameter directly. Our main contributions include the introduction of an exploration schedule and a regret analysis of the proposed algorithm. We provide the convergence rate of the policy parameter to the optimal one, and prove that the algorithm achieves a regret bound of  $O(N^{\frac{3}{4}})$ up to a logarithmic factor, where $N$ is the number of learning episodes. We conduct a simulation study to validate the theoretical results and demonstrate the effectiveness and reliability of the proposed algorithm. We also perform numerical comparisons between our method and those of the recent model-based stochastic LQ RL studies adapted to the state- and control-dependent volatility setting, demonstrating a better performance of the former in terms of regret bounds.
\end{abstract}



\section{Introduction}
	\label{section_introduction}

    \setcounter{equation}{0}
    \setcounter{theorem}{0}
    
    \setcounter{algorithm}{0}

	Linear--quadratic (LQ) control, where the system dynamics are linear in the state and control variables while the rewards are quadratic in them,  takes up a center stage in classical model-based control theory when the model parameters are assumed to be given and known. The reason is twofold: LQ control can be solved explicitly and elegantly, and it can be used to approximate more complicated nonlinear control problems. 
	Detailed theoretical accounts in the continuous-time setting can be found in \cite{anderson2007optimal} for deterministic control (i.e., dynamics are described by ordinary differential equations) and in \cite{YZbook} for stochastic control (i.e., dynamics are governed by stochastic differential equations).
	
	Many real-world applications often present themselves with partially known or entirely unknown environments. Specifically in the LQ context, one may know that a problem is {\it structurally} LQ, namely the system responds linearly to state and control whereas the reward is quadratic (e.g. a variance is involved) in these variables, yet {\it without knowing some or any of the model parameters}. The so-called plug-in method has been traditionally used to solve such a problem, namely, one first estimates the model parameters based on observed data and then plugs in the estimated parameters and applies the classical optimal control theory to derive the solutions. Such an approach is {\it model-based/-driven} because it takes learning the model as its core mission. It is well known, however, that the plug-in method has significant drawbacks, especially in that optimal controls are typically very sensitive to the model parameters, yet estimating some of the parameters accurately when data are limited is a daunting, sometimes impossible,  task (e.g., the return rate of a stock \citep{merton1980estimating, luenberger1998investment}).
	
	Reinforcement learning (RL) has been developed to tackle complex control problems in largely unknown environments. Its successful applications range from strategic board games such as chess and Go \citep{silver2016mastering,silver2017mastering} to robotic systems \citep{gu2017deep,khan2020systematic}. However, RL has been predominantly studied for discrete-time, Markov decision processes (MDPs) with discrete state and control spaces, even though most real-life applications are inherently continuous-time with continuous state space and possibly continuous control space (e.g., autonomous driving, stock trading, and video game playing). More importantly, while one can turn a continuous-time problem into a discrete-time MDP upfront by time and state discretization, such an 
	approach is very sensitive to time step size and performs poorly with small time steps
	\cite{munos2006policy,tallec2019making,park2021time}.
	
	While there were studies directly on continuous-time RL, these had been  rare and far between \citep{baird1994reinforcement,doya2000reinforcement,vamvoudakis2010online,lee2021policy,
		kim2021hamilton} up to just recent years, overall lacking a systematic and unified theory. Starting with \cite{wang2020reinforcement} that introduces an entropy-regularized relaxed control framework for continuous-time RL, a series of subsequent papers \citep{jia2021policy, jia2021policypg, jia2023q} develop theories on policy evaluation, policy gradient, and $q$-learning respectively within this framework.
	This strand of research is characterized by focusing  on learning the optimal control policies directly without attempting to estimate or learn the model parameters,
	underlining a model-free (up to the unknown dynamics being governed by diffusion processes) and data-driven approach.
	The mathematical foundation of the entire theory is the martingale property of certain stochastic processes, the enforcement of which naturally leads to various temporal difference and actor--critic type of RL algorithms to train and learn $q$-functions, value functions, and optimal (stochastic) policies. Subsequently, there has been active follow-up research with various extensions and applications; see, e.g. \cite{huang2022achieving,dai2023learning,wang2023reinforcement,frikha2023actor,
		wei2023continuous,huang2024mean}. 
	
	A crucial question in RL is the convergence and regret bounds of RL algorithms that provide theoretical guidance and guarantee their effectiveness and reliability.
	For LQ problems, such theoretical results exist, for example, for deterministic systems \citep{bradtke1992reinforcement,fazel2018global,malik2019derivative} as well as systems with identically and independently distributed noises \citep{abbasi2011regret,abeille2018improved,
		cohen2018online,cohen2019learning,hambly2021policy,wang2021global,
		yang2019provably,zhou2023single,chen2023global,simchowitz2020naive,cassel2021online}, 
	covering  finite-horizon, infinite-horizon, and ergodic cases. These studies are nevertheless all for discrete-time models, with control not affecting the level of randomness in the state dynamics. Some of them, e.g. \cite{abbasi2011regret,abeille2018improved,hambly2021policy,wang2021global,zhou2023single}, design their algorithms based on policy gradient. However, the gradient representations therein rely on estimations of the drift parameters; hence, the methods are essentially model-based. In addition, the semidefinite programming formulation in \citep{cohen2018online,cohen2019learning} does not seem applicable to continuous-time systems.
	
	The algorithm proposed and analyzed in the present paper belongs to the general class of actor--critic algorithms originally put forward by \cite{konda1999actor}. Such algorithms for discrete-time systems have been studied in \cite{wu2020finite,xu2021doubly,cen2022fast}, and in particular, for ergodic LQ problems in \cite{yang2019provably,chen2023global} and for episodic linear MDPs in \cite{cai2020provably,zhong2024theoretical}. The ``optimal'' regret of these algorithms is mostly of the order $O(\sqrt{N})$, where $N$ is the number of episodes or timesteps. However, it is unclear whether they still work for the diffusion case where the volatility also depends on state and control. For general continuous-time diffusion environments, however, the aforementioned series of papers (\cite{wang2020reinforcement, jia2021policy, jia2021policypg, jia2023q}) and their follow-up study have not addressed the problems of convergence and regret bounds. Assuming oracle access to the policy gradient, \cite{giegrich2024convergence} establishes the exponential convergence of the policy gradient methods for the entropy regularized LQ problem. However, in our setting, the policy gradient needs to be estimated by samples using a model-free approach, leading to significant difficulty in analysis and losses in efficiency compared to the exact policy gradient. In sum, a model-free regret analysis
	remains a highly significant yet challenging open question due to  the stochastic approximation type of algorithms involved.
	
	Recently, there has been some progress on regret analysis for continuous-time stochastic LQ RL in \citep{basei2022logarithmic,szpruch2024optimal} that achieves sublinear regrets of their respective algorithms. Both papers assume that the diffusion coefficients are {\it constant} (independent of state and control), which is vital for their approaches to work. Again, in essence, these works are model-based because they apply either least-square or Bayesian methods to estimate model parameters and use the corresponding estimation errors to deduce the regret bounds. In particular, there is an intrinsic caveat of this approach applied to LQ problems: when estimating the model parameters, an unidentifiability issue may  arise (from optimal control policies being linear feedbacks of
	the state) and a special exploratory stage has to be carefully designed to avoid it; see e.g., \cite{szpruch2021exploration,basei2022logarithmic,szpruch2024optimal,xu2024regret}.
	In addition, to get one estimate of the model parameters, the required number of episodes and that of timesteps within each episode increase exponentially over the learning process, adding substantial computational and memory costs.\footnote{For example, given a total $N$ episodes of the learning budget, the algorithm in \cite{basei2022logarithmic} suggests to break it into several epoches, with the $k$-th epoch containing $2^k$ episodes and $2^{k/2}$ timesteps and to update a policy after having applied it for an epoch of episodes. As a result, although the total regret is proved to be $O(\log N)$, the storage and computational complexity increase exponentially in the number of epoches (which is the number of times of updating the policy parameters).}

	This paper endeavors to design an RL algorithm with a provable sublinear regret for a class of stochastic LQ RL problems in the model-free framework of \cite{wang2020reinforcement,jia2021policy, jia2021policypg, jia2023q}. We allow the diffusion coefficients to depend on both state and control, the latter being of particular practical significance (e.g. the wealth equation in continuous-time finance \citep{zhou2000continuous}). Indeed, this type of stochastic LQ problems have led to a very active research area called ``{\it indefinite} stochastic LQ control" in the classical, model-based literature, starting from \citep{chen1998stochastic, rami2000linear}.

	\paragraph{Main Contributions} In this paper, we propose a policy gradient based algorithm to solve a special class of continuous-time, finite-horizon stochastic  LQ problems under the model-free, episodic RL setting, where the state processes are one dimensional and there is no running control award. Our main contributions are
	\begin{enumerate}
		\item[(1)] We provide a convergence and regret analysis when the volatility of the state process is affected by both state and control. The regret is upper bounded by the order of $O(N^{\frac{3}{4}})$ (up to a logarithmic factor), where $N$ is the number of episodes. While it may not yet be the best regret bound, to our best knowledge, it is the first sublinear regret result obtained in the entropy-regularized exploratory framework of \cite{wang2020reinforcement}, with state- and action-dependent volatility.
		\item[(2)] We take a model-free approach to develop our algorithm, and base our analysis on the stochastic approximation scheme. In particular, the policy gradient in this paper is a ``model-free gradient'' instead of a ``model-based gradient'' commonly taken in discrete-time RL. As a result, we do not need to estimate model primitives in the entire analysis, circumventing the issues discussed earlier arising from estimating/learning those model parameters.
		\item[(3)] We propose a novel exploration schedule. Note that stochastic policies are considered in this paper for both conceptual and technical reasons. Conceptually, stochastic policies reach more action areas otherwise not necessarily explored by deterministic policies. Technically, we apply the policy gradient method developed in \cite{jia2021policypg} that works only for stochastic policies.  Gaussian exploration policies are shown to be optimal in achieving the ideal balance between exploration and exploitation, whose variance represents the level of exploration. We propose a decreasing schedule of variances for the Gaussian exploration over iterations, guided by the desired regret bound.
		
	\end{enumerate}

	The remainder of the paper is structured as follows. Section 2 formulates the problem and provides some preliminary results necessary for the subsequent development. Section 3 describes and explains the steps leading to our RL algorithm. Section 4 presents the main theoretical results on convergence and a regret bound of the proposed algorithm.  Section 5 reports the results of numerical experiments. Finally, Section 6 concludes. The appendix contains the description of the numerical experiments and detailed proofs of statements are available at the full online version \url{https://arxiv.org/pdf/2407.17226}.

	\section{Problem Formulation and Preliminaries}

    \setcounter{equation}{0}
    \setcounter{theorem}{0}
    
    \setcounter{algorithm}{0}

	\subsection{Classical Stochastic LQ Control}
	We begin by recalling the classical stochastic LQ control formulation and main results.  Denote by \(x^u=\{x^u(t) \in \mathbb{R} : 0 \leq t \leq T\}\) the state process under a control process \(u=\{u(t) \in \mathbb{R}^l : 0 \leq t \leq T\}\), whose dynamics are described by the following stochastic differential equation (SDE):
	\begin{equation}
		\label{eq_classical_dynamics}
		\mathrm{d}x^u(t) = (A x^u(t) + B^\top u(t)) \mathrm{d}t + \sum_{j=1}^{m}(C_j x^u(t) + D_j^\top u(t)) \mathrm{d}W^{(j)}(t),
	\end{equation}
	where \(A\) and \(C_j\) are scalars, while \(B\) and \(D_j\) are \(l \times 1\) vectors. The initial state is \(x^u(0)=x_0 \neq 0\), and \(W=\{(W^{(1)}(t),\cdots,W^{(m)}(t))^\top \in \mathbb{R}^m : 0 \leq t \leq T\}\) is an $m$-dimensional standard Brownian motion.
	
	The goal of the control problem is to choose a control process $u$ to maximize the expected value of a quadratic objective functional:
	\begin{equation}
		\label{eq_classical_lq}
		\max_{u} \mathbb{E} \left[ \int_0^T -\frac{1}{2} Q x^u(t)^2 \mathrm{d}t - \frac{1}{2} H x^u(T)^2 \right],
	\end{equation}
	where \(Q \geq 0\) and \(H \geq 0\) are given scalar weighting parameters.
	One can define the optimal value function
	\begin{equation}
		\label{eq_value_optimal}
		V_{CL}(t,x)=\max_{u} \mathbb{E} \left[ \int_t^T -\frac{1}{2} Q x^u(s)^2 \mathrm{d}s - \frac{1}{2} H x^u(T)^2 \Big|x^u(t) = x\right].
	\end{equation}
	
	
	While the existing works on stochastic LQ RL assume the diffusion coefficient to be a constant, control- and state-dependent diffusion terms appear in many applications.
	On the other hand, the state is one-dimensional and running control reward is absent in our problem,  which are crucial assumptions for our approach to work. This class of problems cover important applications such as the mean--variance portfolio selection \citep{zhou2000continuous,dai2023learning,huang2022achieving}. 
	
	If the model parameters \( A \), \( B \), \( C_j \), \( D_j \), \( Q \), and \( H \) are all known with the assumption that  $\Sj D_j D_j^\top$ is positive definite, this problem can be solved explicitly as detailed in \citep[Chapter 6]{YZbook}. Specifically, the  optimal value function and optimal feedback control policy are respectively
	\begin{equation}
		\label{eq_classical_value_function}
		V_{CL}(t,x)=-\frac{1}{2}\biggl[ \frac{Q}{\Lambda} + (H-\frac{Q}{\Lambda})e^{\Lambda (t-T)} \biggl]x^2,
	\end{equation}
	\begin{equation}
		\label{eq_classical_policy}
		u_{CL}(t,x) = -(\sum_{j=1}^{m} D_j D_j^\top)^{-1}(B + \sum_{j=1}^{m}C_jD_j)x,
	\end{equation}
	where
	\[
	\begin{aligned}
		\Lambda &= -2A + 2B^\top \Mj B + 4B^\top \Mj \Vj \\
		&- \sum_{j=1}^m C_j^2 + 2\Vj^\top \Mj \Vj \\
		&- \sum_{j=1}^m \biggl(D_j^\top \Mj B + D_j^\top \Mj \Vj\biggl)^2.
	\end{aligned}
	\]
	
	
	We remark that if \(x\) is of a higher dimension and/or there is a control running reward, then the optimal policy will depend explicitly on the solution of the differential Riccati equation and hence become time-dependent \citep[Chapter 6]{YZbook}, for which our current method fails. \footnote{For example, when \(x\) is higher dimensional, the optimal control is \(u_{CL}(t,x)=-(\sum_{j=1}^m D_j^\top \frac{\partial V_{CL}(t,x)}{\partial xx} D_j)^{-1} \left(B^\top \frac{\partial V_{CL}(t,x)}{\partial x} + \sum_{j=1}^m D_j^\top \frac{\partial V_{CL}(t,x)}{\partial xx} C_j x\right)\), which becomes time-dependent.}
	
	\subsection{RL Theory for LQ}
	\label{subsection_rl_lq}

	In most real-life problems, it is often unrealistic to assume precise knowledge of the parameters such as \(A\), \(B\), \(C_j\), and \(D_j\). These problems call for RL which differs fundamentally from the traditional estimate-then-optimize methods. The essence of RL is to strike an exploration--exploitation balance by strategically exploring the unknown environment \citep{sutton2018reinforcement}. To achieve this, RL employs randomized controls to capture exploration where control processes \(u\) are sampled from a process \(\pi=\{\pi(\cdot,t) \in \mathcal{P}(\mathbb{R}^l) : 0 \leq t \leq T\}\) of  probability distributions with $\mathcal{P}(\mathbb{R}^l)$ being the space of all probability density functions over $\mathbb{R}^l$, and adds an entropy term in the objective function to encourage exploration.   
	Such an entropy regularization is linked to soft-max approximation and Boltzmann exploration \citep{haarnoja2018soft, ziebart2008maximum}.
	\cite{wang2020reinforcement} is the first to present a rigorous mathematical formulation
	of entropy regularized RL for (continuous-time) controlled diffusion processes.
	
	Following \cite{wang2020reinforcement}, under a given randomized control $\pi$ the dynamic of the LQ RL satisfies SDE:
	\begin{equation}
		\mathrm{d}x^\pi(t) = \widetilde{b}(x^\pi(t),\pi(\cdot,t)) \mathrm{d}t + \Sj \widetilde{\sigma}_j(x^\pi(t),\pi(\cdot,t)) \mathrm{d}W^{(j)}(t),
		\label{rl_dynamics}
	\end{equation}
	where 
	\begin{equation}
		\widetilde{b}(x,\psi) := A x + B^\top \int_{\mathbb{R}^l} u \psi(u) \mathrm{d}u,
		\label{b_tilde}
	\end{equation}
	\begin{equation}
		\widetilde{\sigma}_j(x,\psi) := \sqrt{\int_{\mathbb{R}^l} (C_j x + D_j^\top u)^2 \psi(u) \mathrm{d}u},\;\;(x,\psi)\in \mathbb{R}\times \mathcal{P}(\mathbb{R}^l).
		\label{sigma_tilde}
	\end{equation}
	The entropy-regularized value function of $\pi$ is
	\begin{equation}
		\label{eq_value_function_pi}
		J(t, x; \pi) = \mathbb{E} \left[ \int_t^T \left(-\frac{1}{2}Qx^\pi(s)^2 + \gamma p^\pi(s) \right) \mathrm{d}s - \frac{1}{2}Hx^\pi(T)^2 \Big| x^\pi(t)=x \right],
	\end{equation}
	where $p^\pi(t) = -\int_{\mathbb{R}^l} \pi(t,u) \log \pi(t,u) \mathrm{d}u$ is the differential entropy of $\pi$  and \(\gamma \geq 0\), known as the temperature parameter, is the weight on exploration.
	The optimal value function is then
	\begin{equation}
		\label{eq_rl_obj}
		V(t,x) =    \max_{\pi} J(t, x; \pi).
	\end{equation}
	
	By standard stochastic control theory, the value function \(V\)  satisfies the HJB equation:
	
	\[
	\begin{aligned}
		V_t + \max_{\pi\in \mathcal{P}(\mathbb{R}^l)} \biggl\{ \int_{\mathbb{R}^l} \biggl[ &  V_x (A x + B u) + \frac{1}{2} \sum_{j=1}^m \left( (C_j x + D_j u)^2 V_{xx} \right) - \frac{1}{2} Q x^2 \\
		&- \gamma \pi(u) \log \pi(u) \biggl] \pi(u) \mathrm{d}u \biggl\} = 0,\quad V(T,x)=-\frac{1}{2}Hx^2.
	\end{aligned}
	\]
	
	The verification theorem then yields the optimal policy  
	{\small \[
		\begin{aligned}
			\pi(u &\mid t, x) = &\\
			&\mathcal{N}\left(u \Big | -(\sum_{j=1}^{m} D_j D_j^\top V_{xx})^{-1}(B V_x + \sum_{j=1}^{m}C_jD_jV_{xx} x), \gamma (\sum_{j=1}^{m} D_j D_j^\top V_{xx})^{-1} \right),
		\end{aligned}
		\]}
	where $\mathcal{N}(\cdot|\mu,\Sigma)$ is the multivariate Gaussian density with mean $\mu$ and covariance $\Sigma$.
	Plugging this back to the HJB equation, together with the ansatz
	\begin{equation}
		\label{eq_rl_value_function}
		V(t, x) = -\frac{1}{2}k_1(t)x^2 + k_3(t)
	\end{equation}
	where $k_1>0$ and $k_3$ are certain functions of $t$, we obtain the optimal policy
	{\small \begin{equation}
			\label{eq_rl_policy}
			\pi(u \mid t, x) = \mathcal{N}\left(u \Big | -(\sum_{j=1}^{m} D_j D_j^\top)^{-1}(B + \sum_{j=1}^{m}C_jD_j)x, \frac{\gamma}{k_1(t)} \Mj\right).
	\end{equation}}
	Moreover, the resulting HJB equation yields that $k_1$ and $k_3$ satisfy the following ODE:
	\[\frac{k_1^\prime(t)}{2}=-(A+B^\top \bar{\mu}) k_1(t)-\frac{k_1(t)}{2}\Sj \biggl( C_j^2+2C_jD_j^\top\bar{\mu} + D_j^\top \bar{\mu}\bar{\mu}^\top D_j \biggl) - \frac{Q}{2},\quad k_1(T)=H,\] and
	\[k_3^\prime(t) = \frac{k_1(t)}{2} \Sj D_j^\top \Sigma D_j - \frac{\gamma}{2} \log \biggl( (2\pi e)^l \det(\Sigma) \biggl),\quad k_3(T)=0,\]
	where \(\bar{\mu}= -(\sum_{j=1}^{m} D_j D_j^\top)^{-1}(B + \sum_{j=1}^{m}C_jD_j)\) and \(\Sigma=\frac{\gamma}{k_1(t)} \Mj\).\footnote{ Detailed derivations of these results follow analogously from those in \cite{wang2020reinforcement} for the infinite horizon case.}  Note that these {\it theoretical} results cannot be used to compute the solution of the exploratory problem because the model parameters are unknown, yet they reveal the {\it structure} of the solution inherent to LQ RL (i.e. the optimal value function is quadratic in $x$ and optimal stochastic policy is Gaussian) that can be utilized to significantly reduce the complexity of function parameterization/approximation in learning.
	
	Throughout this paper (including the appendices) we use $c$ or its variants for generic constants (depending only on the model parameters \( A \), \( B \), \( C_j \), \( D_j \), \( Q \), \( H \), $x_0$, $T$, $\gamma$, $m$ and $l$) whose values may change from line to line.

	\section{A Continuous-Time RL Algorithm}

    \setcounter{equation}{0}
    \setcounter{theorem}{0}
    
    \setcounter{algorithm}{0}

	This section presents  the steps of designing a continuous-time RL algorithm for solving our  LQ  problem, including function parameterization, policy evaluation and policy gradient.\footnote{As will be explained below, policy evaluation is actually not necessary for the specific LQ problem considered in this paper. However, we still include it in the discussion and in the algorithm for future extension to general problems where policy evaluation is generally needed and indeed a crucial step.} We will introduce various techniques such as exploration scheduling and projection for deriving the convergence rate of the policy parameter and the regret bound. We will also describe time discretization for final implementation. 
	
	\subsection{Function Parameterization}
	\label{subsection_paramtrization}
	
	
	Inspired by \eqref{eq_rl_value_function} and \eqref{eq_rl_policy},  we parameterize   the value function with parameters $\bm{\theta} \in \mathbb{R}^d$:
	\begin{equation}
		\label{value_parametrization}
		{J}(t, x; \bm{\theta}) = -\frac{1}{2}{k}_1(t; \bm{\theta})x^2 + {k}_3(t; \bm{\theta}),
	\end{equation}
	and parameterize the policy with  parameters $\bm{\phi} = (\phi_1, \phi_2)^\top$:
	\begin{equation}
		\label{policy_parametrization}
		{\pi}(u \mid x; \bm{\phi}) = \mathcal{N}(u \mid \phi_1 x, \phi_2),
	\end{equation}
	where $(\phi_1,\phi_2)\in \mathbb{R}^l \times \mathbb{S}^l_{++}$.
	
	Note that \eqref{eq_rl_policy} suggests that the optimal feedback policy is time-dependent, whose variance depends explicitly on $t$. In our parameterization, 
	the time-dependent variance of the Gaussian policies is replaced by a decaying schedule, called an {\it exploration schedule}, of $\phi_2$ as a function of the number of iterations, to be presented shortly.\footnote{An alternative approach is to set a temperature decaying schedule, or equivalently take $\phi_2$ as an endogenous, learnable parameter based on the expression \eqref{eq_rl_policy}. For example, \cite{szpruch2024optimal}
		proposes two different temperature scheduling sequences for their {\it model-based} algorithms respectively, which  determine the variances of the respective policies  in the learning processes that eventually impact the regrets of the algorithms. In this paper, we choose to take a shortcut by taking the entire policy variance as a suitable exogenous tuning sequence  (as shown in Theorem \ref{thm_convergence_and_rate_in_main}), leaving the temperature $\gamma$ as a constant for simplicity.}

	
	
	Henceforth we assume that there are positive constants $c_1,c_2,c_3$ such that $ 1/c_2 \leq {k}_1(t; \bm{\theta}) \leq {c}_2$,  $|{k}_1^\prime(t; \bm{\theta})| \leq {c}_1$ and $|{k}_3^\prime(t; \bm{\theta})| \leq {c}_3$,  for any $0 \leq t \leq T$. These assumptions are consistent with the fact that the corresponding functions satisfy the same conditions
	when the model parameters are known. For practical implementations, in general we choose $ {k}_1(\cdot; \bm{\theta})$  and ${k}_3(\cdot; \bm{\theta})$ to be neural networks that  satisfy these conditions.
	
	\subsection{Policy Evaluation}
	Policy evaluation (PE) is generally a key step in RL to learn the value function of a {\it given} control policy. 
	
	The general continuous-time PE method developed in \citep{jia2021policy} dictates that
	one first parameterizes the value function $J(\cdot, \cdot; \pi)$ and the policy $\pi$ by \eqref{value_parametrization} and \eqref{policy_parametrization} respectively (with a slight abuse of notation), with the corresponding $p^\pi(t)=p(t;\bm{\phi})$, and then updates  $\bm{\theta}$ in an offline learning setting:
	\begin{equation}
		\label{eq_theta_update_without_projection}
		\bm{\theta} \leftarrow \bm{\theta} + \alpha \int_0^T \frac{\partial J}{\partial \bm{\theta}}(t, x(t); \bm{\theta}) \left[ \mathrm{d}J(t, x(t); \bm{\theta}) - \left(\frac{1}{2} Qx(t)^2 \mathrm{d}t - \gamma {p}(t; \bm{\phi}) \mathrm{d}t\right) \right],
	\end{equation}
	where $\alpha$ is the learning rate.
	
	Intriguingly, however, our subsequent theoretical analysis indicates that the convergence and regret results depend only on the bounds (i.e. the constants $c_1$, $c_2$ and $c_3$) of the functions $k_1$ and $k_2$, {\it not} on the specific forms of these functions. This feature is due to the special class of LQ control problems we are tackling. As a result, in our numerical experiments we actually fix a value function (or equivalently $\bm{\theta}$) throughout without updating it. However, we still include the policy evaluation part in our exposition as it is needed in more general cases.

	\subsection{Policy Iteration}
	Having learned the value function associated with a Gaussian policy, the next step is to
	improve the policy by updating $\bm\phi=(\phi_1, \phi_2)^\top$. For $\phi_1$, we employ the continuous-time policy gradient (PG) method established in  \citep{jia2021policypg} to get the following updating rule:
	\begin{equation}
		\label{eq_phi_updates_without_projections}
		\phi_{1} \leftarrow \phi_{1} + \alpha Z_{1}(T),
	\end{equation}
	where $\alpha$ is the learning rate, and
	\begin{equation}
		\label{eq_z1_def1}
		\begin{aligned}
			Z_{1}(s) = \int_{0}^{s} & \left\{\frac{\partial \log \pi}{\partial \phi_1}\left(u(t) \mid x(t); \bm\phi\right)
			\left[\mathrm{d} J\left(t, x(t); \bm\theta \right) - \frac{1}{2}Q x(t)^2 \mathrm{d} t \right.\right. \\
			& \left.\left. + \gamma {p}\left(t, \bm\phi\right) \mathrm{d} t\right] + \gamma \frac{\partial {p}}{\partial \phi_1}\left(t, \bm\phi \right) \mathrm{d} t \right\}, \;\;0\leq s\leq T.
		\end{aligned}
	\end{equation}
	
	As discussed earlier, the other parameter, $\phi_2$, controls the level of exploration. In our algorithm, we set a deterministic schedule of this parameter which decreases to 0 as the number of iterations grows. Specifically, we set \(\phi_{2,n} = \frac{I_l}{b_n}\) where $I_l$ is the identity matrix of dimension $l$ and \(b_n \uparrow \infty\) is specified in Theorem \ref{thm_convergence_and_rate_in_main} below. The order of $b_n$ in iteration $n$ is carefully chosen along with those of the other hyperparameters, such as the learning rates, in order to achieve the desired sublinear regret bound of the RL algorithm.

	\subsection{Projections}
	\label{section_projections}
	
	Our updating rules for the parameters \(\bm\theta\) and \(\bm\phi\) are types of stochastic approximation (SA),  a technique pioneered by \citep{robbins1951stochastic}.
	To tailor the general SA algorithms to our specific requirements—primarily to circumvent issues like extreme state values and unbounded estimation errors—we include projection, a technique originally proposed by \citep{andradottir1995stochastic}. The projection maps do not depend on prior environmental knowledge, allowing our method to remain model-free while ensuring that the learning regions expand to cover the entire parameter space over time.
	
	First, in general ${k}_1(\cdot; \bm\theta)$ and ${k}_3(\cdot; \bm\theta)$ are two fixed function approximators (e.g. neural networks), which can be properly chosen so that  there exist constants $c_1, c_2, c_3$ satisfying
	\[ 1/c_2 \leq {k}_1(t; \bm\theta) \leq {c}_2,  |{k}_1^\prime(t; \bm\theta)| \leq {c}_1, |{k}_3^\prime(t; \bm\theta)| \leq {c}_3, \;\;\; \forall (t,\bm\theta).
	\]
	Fix a constant $c_{\bm\theta}>0$ and define
	\begin{equation}
		\label{eq_projections1}
		K_{\bm\theta}=\left\{ \bm\theta\in \mathbb{R}^d \Big| |\bm\theta| \leq c_{\bm\theta}\right\}.
	\end{equation}
	Note that the definition of $K_{\bm\theta}$ is specific to our special class of LQ problems under consideration -- it is independent of $n$ because the subsequent regret analysis, as it turns out, does not rely on the convergence of $\bm\theta$. Importantly, the hyperparameters $c_1, c_2, c_3$ are determined by the specified  function approximators ${k}_1(\cdot; \bm\theta)$ and ${k}_3(\cdot; \bm\theta)$, without requiring prior knowledge of the model parameters.  (Later in our numerical experiments we will simply set $k_1 \equiv 1$ and $k_3\equiv 0$.)
	
	Next, define
	\begin{equation}
		\label{eq_projections}
		K_{1,n}=\left\{ \phi_{1} \in \mathbb{R}^l \Big| |\phi_{1}|\leq c_{1,n}  \right\},
	\end{equation}
	where $\{c_{1,n}\}$ is an increasing sequence to be specified in Theorem \ref{thm_convergence_and_rate_in_main}. Clearly, $K_{1,n} \subset K_{1,n+1} \subset \cdots$ and $\cup_{n} K_{1,n} = \mathbb R^l$. Finally, for any nonempty, convex and compact set $K$, we define $\Pi_{K}(x):=\arg \min_{y\in K} |y-x|^2$ to be the standard projection mapping of a point $x$ onto  $K$. In our updating rule below, we mainly control the learned policy $\bm\phi_n \in K_{1,n}$ using projection so that it will not grow too fast to an unbounded region.

	With projections, the updating rules for $\bm\theta$ and $\phi_1$ in \eqref{eq_theta_update_without_projection} and \eqref{eq_phi_updates_without_projections} are modified as follows. Let $x_n$ and $u_n$ be the state trajectory and the control trajectory at the $n$-th iteration, obtained from the systems dynamic \eqref{eq_classical_dynamics} with a stochastic policy \eqref{policy_parametrization}, i.e., $
	u_n(t)|x_n(t) \sim \mathcal{N}(\phi_{1,n} x_n(t), \phi_{2,n})$. We update
	\begin{equation}
		\label{eq_theta_update}
		\begin{split}
			\bm\theta_{n+1} \leftarrow \Pi_{K_{\bm\theta}}\biggl( & \bm\theta_n + a_n\int_0^T \frac{\partial J}{\partial \bm\theta}(t,x_n(t);\bm\theta_n) \\
			& \left[\mathrm{d} J\left(t, x_n(t); \bm\theta_n \right) - \frac{1}{2}Qx_n(t)^2 \dd t + \gamma {p}\left(t, \bm\phi_n\right) \mathrm{d} t\right] \biggl),
		\end{split}
	\end{equation}
	
	\begin{equation}
		\label{eq_phi1_update}
		\begin{split}
			\phi_{1, n+1} \leftarrow \Pi_{K_{1,n+1}} \biggl( \phi_{1,n}+a_{n} Z_{1,n}(T) \biggl),
		\end{split}
	\end{equation}
	where
	\begin{equation}
		\label{eq:z1_def1}
		\begin{aligned}
			Z_{1,n}(s) = \int_{0}^{s} \biggl\{ & \frac{\partial \log \pi}{\partial \phi_1}\left(u_n(t) \mid  x_n(t); \bm\phi_n\right)
			\biggl[\mathrm{d} J\left(t, x_n(t); \bm\theta_n \right) - \frac{1}{2}Q x_n(t)^2 \dd t \\
			& + \gamma {p}\left(t, \bm\phi_n\right) \mathrm{d} t\biggl] + \gamma \frac{\partial {p}}{\partial \phi_1}\left(t, \bm\phi_n \right) \mathrm{d} t \biggl\},\;\;0\leq s\leq T.
		\end{aligned}
	\end{equation}

	\subsection{Discretization}
	\label{sec:discretization}
	Our approach for continuous-time RL is characterized by carrying out the entire analysis in the continuous-time setting and discretizing time only at the final implementation stage. There are two reasons for discretization. First, note that in the updating rules \eqref{eq_theta_update} to \eqref{eq:z1_def1}, we implicitly assume a {\it continuous} sampling of control $u_n(t)|x_n(t) \sim \mathcal{N}(\phi_{1,n} x_n(t), \phi_{2,n})$, which is not practical for actual inplementation. Second,
	The iterations in \eqref{eq_theta_update} and \eqref{eq_phi1_update} involve integrals
	that can be computed only by approximated discretized summations as well as the $\dd J$ term that can be approximated by the temporal difference between two consecutive time steps.\footnote{\label{fn:online_version} Therefore, the error due to discretization also has two parts -- one by discretely sampling randomized controls and the other by computing the integrals using finite sums. \cite{szpruch2024optimal} focuses on the first type of error and \cite{basei2022logarithmic} on the second one in their analyses respectively. We address both types of discretization errors in the proof of Theorem \ref{thm_convergence_and_rate_in_main}, as well as in Theorems \ref{thm:phi1_convergence} and \ref{thm_phi_1_rate} in Appendix \ref{section_additional_proof_app} of the full online version of this paper, available at \url{https://arxiv.org/pdf/2407.17226} (which hereafter will be referred to as the “online version”).}
	Therefore, at the $n$th iteration we  discretize the interval \([0, T]\) into uniform time intervals of length \(\Delta t_n\), leading to the following schemes:

	\begin{equation}
		\label{eq_theta_update_discrete}
		\begin{aligned}
			\bm\theta_{n+1} \leftarrow \Pi_{K_{\bm\theta, n+1}}\biggl(& \bm\theta_n+a_n \sum_{k=0}^{\left\lfloor \frac{T}{\Delta t_n} -1 \right\rfloor} \frac{\partial J}{\partial \bm\theta}(t_k,x_n(t_k);\bm\theta_n)\biggl[ J\left(t_{k+1}, x_n(t_{k+1}); \bm\theta_n \right) \\
			&- J\left(t_k, x_n(t_k); \bm\theta_n \right) - \frac{1}{2}Qx_n(t_k)^2 \Delta t_n + \gamma {p}\left(t_k, \bm\phi_n\right) \Delta t_n\biggl] \biggl),
		\end{aligned}
	\end{equation}
	
	\begin{equation}
		\label{eq_phi1_update_discrete}
		\begin{aligned}
			\phi_{1, n+1} \leftarrow &\Pi_{K_{1,n+1}} \biggl( \phi_{1,n} + a_n \sum_{k=0}^{\left\lfloor \frac{T}{\Delta t_n} -1 \right\rfloor} \biggl\{\frac{\partial \log \pi}{\partial \phi_1}\left(u_n(t_k) \mid t_k, x_n(t_k); \bm\phi_n\right)
			\\
			&\biggl[ J\left(t_{k+1}, x_n(t_{k+1}); \bm\theta_n \right) - J\left(t_k, x_n(t_k); \bm\theta_n \right) - \frac{1}{2}Qx_n(t_k)^2 \Delta t_n \\
			&+ \gamma {p}\left(t_k, \bm\phi_n\right) \Delta t_n\biggl]+
			\gamma \frac{\partial {p}}{\partial \phi_1}\left(t_k, \bm\phi_n \right) \Delta t_n \biggl\} \biggl).
		\end{aligned}
	\end{equation}

	\subsection{RL-LQ Algorithm}
	The analysis above leads to the following RL algorithm for the LQ problem:
	
	\begin{algorithm}[H]
		\caption{RL-LQ Algorithm}\label{algo_rl-lq}
		\begin{algorithmic}
			
			\Input
			\Desc{$\bm\theta_0$, $\phi_{1,0}$}{Initial values of trainable parameters for value function and policy.}
			\Desc{$\phi_{2,n}$}{Deterministic sequence of $\phi_{2,n}=\frac{I_l}{b_n}$ specified in Theorem \ref{thm_convergence_and_rate_in_main}.}
			\EndInput
			\\\hrulefill
			
			\For{\texttt{$n = 1$ to $N$}}
			\State Initialize $k=0$, time $t=t_k = 0$, state $x_n(t_k) = x_0$.
			\While {$t < T$}
			\State Generate action $u_n(t_k) \sim \bm\pi\left(\cdot \mid t_{k}, x_n(t_{k}); \bm\phi_n\right)$ following policy \eqref{policy_parametrization} .
			\State Apply action $u_n(t_k)$ and get new state $x_n(t_{k+1})$ by dynamic \eqref{eq_classical_dynamics}.
			\State Update time $t_{k+1} \leftarrow t_k + \Delta t_n$ and $t \leftarrow t_{k+1}$.
			\EndWhile
			\State Collect  whole trajectory $\{ (t_k, x_n(t_k), u_n(t_k)) \}_{k\geq 0}$.
			\State Update  value function parameters $\bm\theta$ using \eqref{eq_theta_update_discrete}.
			\State Update  policy parameter $\phi_1$ using \eqref{eq_phi1_update_discrete}.
			\EndFor
			\\\hrulefill
			
			\Output
			\Desc{$\bm\theta_N, \phi_{1,N}, \phi_{2,N}$}{\qquad Parameters for value function and policy.}
			\EndOutput
		\end{algorithmic}
	\end{algorithm}

	\section{Regret Analysis}
	\label{section_regret_analysis}

        \setcounter{equation}{0}
        \setcounter{theorem}{0}
        
        \setcounter{algorithm}{0}

	This section presents the main result of the paper -- a sublinear regret bound of the RL-LQ algorithm, Algorithm \ref{algo_rl-lq}. For that, we need to first examine the convergence property and convergence rate of the parameter \(\phi_{1,n}\), whose analysis forms the theoretical underpinning of the algorithm.

	\subsection{Convergence  of $\phi_{1,n}$}
	The following theorem shows the convergence and convergence rate of the parameter \(\phi_{1,n}\).
	
	\begin{theorem}
		\label{thm_convergence_and_rate_in_main}
		In Algorithm \ref{algo_rl-lq}, let the hyperparameters $c_1$, ${c}_2$, ${c}_3$ and $\gamma$ be fixed positive constants. Set
		$$a_n  = \frac{\alpha^{\frac{3}{4}}}{(n+\beta)^{\frac{3}{4}}},\;\;
		b_n = 1 \vee \frac{(n+\beta)^{\frac{1}{4}}}{\alpha^{\frac{1}{4}}},\;\;c_{1,n}=1 \vee (\log \log n)^{\frac{1}{6}},\;\; \Delta t_n = T (n+1)^{-\frac{5}{8}}$$
		where $\alpha>0$ and $\beta>0$ are constants. Then,
		\begin{enumerate}[label=(\alph*)]
			\item as $n \rightarrow \infty$, $\phi_{1,n}$ converges almost surely to
			\[\phi_1^* =  -(\sum_{j=1}^{m} D_j D_j^\top)^{-1}(B + \sum_{j=1}^{m}C_jD_j).\]
			\item  for any $n$, $\E [|\phi_{1,n} - \phi_1^*|^2] \leq {c} \frac{(1 \vee \log n)^p (1 \vee \log \log n)^{\frac{4}{3}}}{n^{\frac{1}{2}}}$, for some positive constants ${c}$ and $p$.
		\end{enumerate}
	\end{theorem}
	
	
	\begin{proof}
		To keep our proof easy to follow, we first present the analysis by ignoring both types of discretization error pointed out in Section \ref{sec:discretization}. More precisely, we first assume a {continuous} sampling of control $u_n(t)|x_n(t) \sim \mathcal{N}(\phi_{1,n} x_n(t), \phi_{2,n})$ with $\Delta t_n=0$ and use updating rules \eqref{eq_theta_update}--\eqref{eq:z1_def1}. After we prove the desired conclusion without discretization error, we address the impact of the discretization error and show that it will not change the final conclusion.
		
		Let $x_n=\{x_n(t) : 0\leq t \leq T\}$ be the sample state trajectory in the $n$-th iteration that follows the dynamics:
		\begin{equation}
			\label{eq_dynamics_app}
			\mathrm{d}x_n(t) = (A x_n(t) + B^\top u_n(t)) \mathrm{d}t + \Sj (C_j x_n(t) + D_j^\top u_n(t)) \mathrm{d}W_n^{(j)}(t), \quad 0\leq t \leq T,
		\end{equation}
		where \(W_n=\{(W_n^{(1)}(t),\cdots,W_n^{(m)}(t))^\top \in \mathbb{R}^m : 0 \leq t \leq T\}\) is a standard Brownian motions in the $n$-th iteration, and the policy $u_n(t) \mid x_n(t) \sim \mathcal{N}(\cdot|\phi_{1,n} x_n(t), \phi_{2,n})$ independent of $W_n$.
		
		Recall \( Z_{1}(\cdot) \) defined by \eqref{eq_z1_def1}, and  \( Z_{1,n}(T) \) defined by \eqref{eq:z1_def1} as the value of \( Z_{1}(T) \) at the \( n \)-th iteration.
		The expectation of \(Z_{1,n}(T)\) conditioned on the parameters is denoted by
		\[
		h_1(\phi_{1,n}, \phi_{2,n}; \bm{\theta}_n) = \E[Z_{1,n}(T) \mid \bm{\theta}_n, \bm{\phi}_n],
		\]
		and the noise contained in \({Z}_{1,n}(T)\) is defined as
		\[
		\xi_{1,n} = {Z}_{1,n}(T) - h_1(\phi_{1,n}, \phi_{2,n}; \bm{\theta}_n).
		\]
		Hence, the updating rule for \(\phi_1\) is given by:
		\begin{equation}
			\label{eq_phi1_update_app}
			\phi_{1, n+1} = \Pi_{K_{1, n+1}} (\phi_{1,n} + a_{n}[h_1(\phi_{1,n}, \phi_{2,n}; \bm{\theta}_n) + \xi_{1, n}]).
		\end{equation}
		
		To establish the main results, we need  a series of technical lemmas. Due to page limit, we provide the statements and proofs of these auxiliary results  in Appendices \ref{section_moments_app} - \ref{section_MSE_app} of the online version.
			
			The proof of part (a) follows the similar analysis for the almost sure convergence of the projected stochastic approximation in \citep{andradottir1995stochastic}, whose conditions are relaxed in Theorem \ref{thm:phi1_convergence}. The proof of part (b) follows the similar analysis for the optimal convergence rate of the stochastic approximation in \citep{broadie2011general}, whose conditions are relaxed in Theorem \ref{thm_phi_1_rate}. In short, there are three key steps in the proofs of the two parts. The first is to show that the function $h_1$ satisfies certain conditions, which is given in Appendix \ref{appendix:mean increment}. The second is to prove that the deviation between $Z_{1,n}(T)$ and $h_1$ cannot be too large, which is provided  in Appendix \ref{section_moments_app}. The final step is to verify the compatibility between  the learning rate sequences $a_n$ and the projected regions $b_n,c_n$, described in Theorem \ref{thm:phi1_convergence}.
			
			After having established the desired convergence results for ideally sampled \\
            process, we now address the impact of the discretization error.
		The difference between \eqref{eq_phi1_update_discrete} and \eqref{eq_phi1_update} is that the latter assumes continuous (over time) sampling of the controls $u_n(t)$ and exact calculation of the integral term. Therefore, using the former to replace the latter affects both its mean and variance (conditioned on $\bm\theta_n,\bm\phi_n$). More precisely, denote by $\hat Z_{1,n}(T)$  the finite-sum part in \eqref{eq_phi1_update_discrete}, while keeping the definition of the function $h_1$. Then the updating rule for $\phi_1$ in \eqref{eq_phi1_update_discrete} can be written as
		\begin{equation}
			\label{eq_phi1_update_app_discrete}
			\phi_{1, n+1} = \Pi_{K_{1, n+1}} (\phi_{1,n} + a_{n}[h_1(\phi_{1,n}, \phi_{2,n}; \bm{\theta}_n) + \hat\xi_{1, n}]),
		\end{equation}
		with
		\[ \hat\xi_{1,n} = \hat{Z}_{1,n}(T) - h_1(\phi_{1,n}, \phi_{2,n}; \bm{\theta}_n). \]
		Observe that the updating rule \eqref{eq_phi1_update_app_discrete} has the same form as \eqref{eq_phi1_update_app}; hence our previous analysis can be conducted in parallel if we can show that
		\[ \beta_{1,n} = \E\left[ \hat\xi_{1,n} \big| \bm\theta_n,\bm\phi_n  \right]\text{ and } \operatorname{Var}\left[ \hat\xi_{1,n} \big| \bm\theta_n,\bm\phi_n  \right]  \]	
		satisfy the conditions in Theorems \ref{thm:phi1_convergence} and \ref{thm_phi_1_rate} of the online version.
		Similar problems for analyzing time discretization have been studied in the numerical SDE literature (cf. \cite{kloedennumerical}), and recently in the continuous-time RL context (cf. \cite{basei2022logarithmic,szpruch2024optimal,joz2025}). In particular, following the parallel steps as in the proof of \cite[Proposition 5.4]{joz2025} (we omit the details here), we  obtain
		\[\beta_{1,n} \leq C(\bm\phi_n) \Delta t_n, \;\;\;
		\operatorname{Var}\left[ \hat\xi_{1,n} \big| \bm\theta_n,\bm\phi_n  \right] \leq C(\bm\phi_n),\]
		where
		\begin{equation*}
			C(\bm\phi_n) \leq cb_n(|\phi_{1,n}|^8+1)\exp\{c |\phi_{1,n}| ^6\}.
		\end{equation*}
		
		Therefore, by taking $\Delta t_n = T (n+1)^{-\frac{5}{8}}$, we derive from Theorems \ref{thm:phi1_convergence} and \ref{thm_phi_1_rate} of the online version that \(|\beta_{1,n}|\) is of the order \( n^{-\frac{3}{8}} \). This concludes the proof.
	\end{proof}
	
	Theorem \ref{thm_convergence_and_rate_in_main} ensures the convergence of the learned policy. Moreover, it is a prerequisite for deriving the regret bound of Algorithm \ref{algo_rl-lq}.

	\subsection{Regret Bound}
	A regret bound measures the cumulative derivation (over episodes) of the value functions of the learned policies from the oracle optimal value function. A sublinear regret bound guarantees an almost optimal performance of the RL policy in the long run.

	Denote
	\begin{equation}
		\label{jbar}
		\Bar{J}(\phi_1, \phi_2) = \mathbb{E} \left[ \int_0^T \left(-\frac{1}{2}Qx^\pi(s)^2  \right) \mathrm{d}s - \frac{1}{2}Hx^\pi(T)^2 \Big| x^\pi(0)=x_0 \right],
	\end{equation}
	where $\pi = \mathcal{N}(\cdot|\phi_1 x, \phi_2)$.
	
	So $\Bar{J}(\phi_1, \phi_2)$ is the value of the  Gaussian policy $\mathcal{N}(\cdot|\phi_1 x, \phi_2)$ assessed using the {\it original} objective function (i.e. one {\it without} the entropy regularization term). Clearly, $\Bar{J}(\phi_1^*, 0)$ is the oracle value of the original problem.
	
	\begin{theorem}
		\label{thm_regret_in_main}
		Under the assumptions of Theorem \ref{thm_convergence_and_rate_in_main}, applying Algorithm \ref{algo_rl-lq} results in a cumulative regret bound over \(N\) episodes given by:
		\[\sum_{n=1}^{N}\E [\Bar{J}(\phi_1^*, 0) - \Bar{J}(\phi_{1,n}, \phi_{2,n})] \leq c^\prime + c N^{\frac{3}{4}}(\log N)^{\frac{p+1}{2}} (\log\log N)^{\frac{2}{3}},\]
		where \(c\) and \(c^\prime\) are positive constants independent of \(N\), and $p$ is the same constant appearing in Theorem \ref{thm_convergence_and_rate_in_main}.
	\end{theorem}

	\begin{proof}
		By Lemma \ref{lemma_j_f_g_main} of the online version, we can write
		\begin{equation}
			\label{eq_all_terms}
			\begin{aligned}
				&\Bar{J}(\phi_1^*, 0) - \Bar{J}(\phi_{1,n}, \phi_{2,n}) \\
				=& f(a(\phi_1^*))  - [f(a(\phi_{1,n})) + (\Sj D_j^\top \phi_{2,n} D_j) g(a(\phi_{1,n}))]\\
				=& -(\Sj D_j^\top \phi_{2,n} D_j)g(a(\phi_1^*)) + [f(a(\phi_1^*)) - f(a(\phi_{1,n}))] \\
				&+ (\Sj D_j^\top \phi_{2,n} D_j)[g(a(\phi_1^*)) - g(a(\phi_{1,n}))].
			\end{aligned}
		\end{equation}
		Because $\phi_{2,n}=\frac{I_l}{b_n}$ and $g<0$ (by Lemma \ref{lemma_j_properties} of the online version), we have
		
		\begin{equation}
			\label{eq_term_3}
			\begin{aligned}
				\E[-(\Sj D_j^\top \phi_{2,n} D_j)g(a(\phi_1^*))] = -g(a(\phi_1^*)) \frac{D}{b_n} \leq \frac{c}{n^{\frac{1}{4}}},
			\end{aligned}
		\end{equation}
		where $D=(\Sj D_j^\top D_j)$ and $c>0$ is independent of $n$.
		
		
		Next, noting by \eqref{eq_a_phi_1} of the online version that $a(\phi_1)$ is a quadratic function of  $\phi_1$ and  $\phi_1^*$ is its minimizer, we have $|a(\phi_1) - a(\phi_1^*)| \leq \Bar{c}_1 |\phi_1 - \phi_1^*|^2$, where $\Bar{c}_1>0$ is a constant that depends on the model primitives. Furthermore, it follows from \eqref{eq_a_phi_1} of the online version and \eqref{eq_projections} that $|a(\phi_{1,n})|\leq\Bar{c}_2(1+c_{1,n}^2)$ for some constant $\Bar{c}_2>0$. In addition, by the monotonicity of the functions $f$ and $g$ and the assumptions on the choices of the hyperparameters (noting Remark \ref{remark_example_1} of the online version), we have
		
		\[
		|f(a(\phi_{1,n}))| \leq |f(\Bar{c}_2(1+c_{1,n}^2))|\leq c (1+e^{\Bar{c}_2(1+c_{1,n}^2)T})(1+\frac{1}{c_{1,n}^2})\leq\Bar{c}_{3}\log n,
		\]
		\[
		|g(a(\phi_{1,n}))| \leq |g(\Bar{c}_2(1+c_{1,n}^2))|\leq c (1+e^{\Bar{c}_2(1+c_{1,n}^2)T})(1+\frac{1}{c_{1,n}^4})\leq\Bar{c}_{4}\log n,
		\]
		where $\Bar{c}_3>0$ and $\Bar{c}_4>0$ are  constants independent of $n$.
		
		Furthermore, it follows from Lemma \ref{lemma_j_properties} of the online version that for a given fixed $\delta > 0$, the inequalities $|f'(a(\phi_1))| \leq \Bar{c}_5(\delta)$ and $|g'(a(\phi_1))| \leq \Bar{c}_6(\delta)$ hold for any $\phi_1$ satisfying $|\phi_1 - \phi_1^*| \leq \delta$, where $\Bar{c}_5(\delta)>0$ and $\Bar{c}_6(\delta)>0$ are constants that depend on $\delta$. Theorem \ref{thm_phi_1_rate} of the online version yields
		\[
		\E [|\phi_{1, n+1} - \phi_1^*|^2] \leq \Bar{c}_7 \frac{(1 \vee \log n)^{p} (1 \vee \log \log n)^{\frac{4}{3}}}{n^{\frac{1}{2}}},
		\]
		where $\Bar{c}_7 > 0$ is a constant. Now, we consider a positive sequence
		\[\delta_{1,n}=\biggl(\frac{|f(a(\phi_1^*))| + \Bar{c}_{3} \log n}{\Bar{c}_5(\delta) \Bar{c}_1} \frac{\Bar{c}_7 (1\vee \log n)^{p} (1\vee \log\log n)^{\frac{4}{3}}}{n^{\frac{1}{2}}}\biggl)^{\frac{1}{4}},\;\;n=1,2, \cdots.\]
		It is straightforward to see that $\delta_{1,n} \rightarrow 0$ as $n\rightarrow\infty$.
		
		Thus there exists the finite $n_1=\inf{\{n^\prime\in\mathbb{N}: \delta_{1,n}<\delta \text{ for all } n\geq n^\prime\}}$. Denote $\delta_n=\delta$ for $n<n_1$ and $\delta_n=\delta_{1,n}$ for $n\geq n_1$. When $n\geq n_1$, we deduce
		

		\begin{align}
			\label{eq_term_1}
			&\E[f(a(\phi_1^*)) - f(a(\phi_{1,n}))]\\
			=&\E[f(a(\phi_1^*))-f(a(\phi_{1,n}))]\mathbf{1}_{\{|\phi_{1,n}-\phi_1^*|\leq\delta_n\}} + \E[f(a(\phi_1^*))-f(a(\phi_{1,n}))]\mathbf{1}_{\{|\phi_{1,n}-\phi_1^*|>\delta_n\}} \notag \\
			=& \int_{|\phi_{1,n}-\phi_1^*|\leq\delta_n} f(a(\phi_1^*))-f(a(\phi_{1,n})) \dd \mathbb{P} + \int_{|\phi_{1,n}-\phi_1^*|>\delta_n} f(a(\phi_1^*))-f(a(\phi_{1,n})) \dd \mathbb{P} \notag\\
			=& \int_{|\phi_{1,n}-\phi_1^*|\leq\delta_n} f^\prime(a(\Tilde{\phi}_{1,n})) (a(\phi_1^*) - a(\phi_{1,n})) \dd \mathbb{P} + \int_{|\phi_{1,n}-\phi_1^*|>\delta_n} f(a(\phi_1^*))-f(a(\phi_{1,n})) \dd \mathbb{P}\notag\\
			\leq& \Bar{c}_5(\delta) \Bar{c}_1 \delta_n^2 + (|f(a(\phi_1^*))| + |f(a(\phi_{1,n}))|) \mathbb{P}(|\phi_{1,n}-\phi_1^*|>\delta_n)\notag\\
			\leq& \Bar{c}_5(\delta) \Bar{c}_1 \delta_n^2 + \frac{|f(a(\phi_1^*))| + \Bar{c}_{3} \log n}{\delta_n^2} \E[|\phi_{1,n}-\phi_1^*|^2]\notag\\
			\leq& \Bar{c}_5(\delta) \Bar{c}_1 \delta_n^2 + \frac{|f(a(\phi_1^*))| + \Bar{c}_{3} \log n}{\delta_n^2} \frac{\Bar{c}_7 (1\vee \log n)^{p} (1\vee \log\log n)^{\frac{4}{3}}}{n^{\frac{1}{2}}}\notag\\
			=& 2 \sqrt{\Bar{c}_7\Bar{c}_5(\delta) \Bar{c}_1 (|f(a(\phi_1^*))| + \Bar{c}_{3} \log n)} \frac{(1\vee \log n)^{\frac{p}{2}} (1\vee \log\log n)^{\frac{2}{3}}}{n^{\frac{1}{4}}}\notag,
		\end{align}
		
		where the third equality follows from the  mean--value theorem and the fact that $\Tilde{\phi}_{1,n} \in \{\phi_1 \in \mathbb{R}^l:|\phi_1-\phi_1^*|<|\phi_{1,n}-\phi_1^*|\}$ satisfies $f(a(\phi_1^*))-f(a(\phi_{1,n})) = f^\prime(a(\Tilde{\phi}_{1,n})) \\(a(\phi_1^*) - a(\phi_{1,n}))$.
		
		When $n<n_1$, by the same argument as in \eqref{eq_term_1} with $\delta_n$ replaced by $\delta$, we have,
		\begin{equation}
			\begin{aligned}
				\label{eq_term_1_2}
				&\E[f(a(\phi_1^*)) - f(a(\phi_{1,n}))] \\
				\leq& \Bar{c}_5(\delta) \Bar{c}_1 \delta^2 + \frac{|f(a(\phi_1^*))| + \Bar{c}_{3} \log n}{\delta^2} \frac{\Bar{c}_7 (1\vee \log n)^{p} (1\vee \log\log n)^{\frac{4}{3}}}{n^{\frac{1}{2}}}.
			\end{aligned}
		\end{equation}
		
		Similarly, we consider another positive sequence
		\[\delta_{2,n}=\biggl(\frac{|g(a(\phi_1^*))| + \Bar{c}_4 \log n}{\Bar{c}_6(\delta) \Bar{c}_1} \frac{\Bar{c}_7 (1\vee \log n)^{p} (1\vee \log\log n)^{\frac{4}{3}}}{n^{\frac{1}{2}}}\biggl)^{\frac{1}{4}},\;\;n=1,2, \cdots.\]
		Because $\delta_{2,n} \rightarrow 0$ as $n\rightarrow\infty$, there exists the finite
		\[n_2=\inf{\{n^\prime\in\mathbb{N}: \delta_{2,n}<\delta \text{ for all } n\geq n^\prime\}}.\] Set $\delta_n^\prime=\delta$ for $n<n_2$ and $\delta_n^\prime=\delta_{2,n}$ for $n\geq n_2$. When $n\geq n_2$, we have
		
		\begin{align}
			\label{eq_term_2}
			&\E[(\Sj D_j^\top \phi_{2,n} D_j)(g(a(\phi_1^*))-g(a(\phi_{1,n})))] \\
			=& \E[(\Sj D_j^\top \phi_{2,n} D_j)(g(a(\phi_1^*))-g(a(\phi_{1,n})))]\mathbf{1}_{\{|\phi_{1,n}-\phi_1^*|\leq\delta_n^\prime\}} + \E[(\Sj D_j^\top \phi_{2,n} D_j)(g(a(\phi_1^*))-g(a(\phi_{1,n})))]\mathbf{1}_{\{|\phi_{1,n}-\phi_1^*|>\delta_n^\prime\}} \notag\\
			=& \frac{D}{b_n}\int_{|\phi_{1,n}-\phi_1^*|\leq\delta_n^\prime} (g^\prime(a(\check{\phi}_{1,n})) (a(\phi_1^*) - a(\phi_{1,n}))) \dd \mathbb{P} + \frac{D}{b_n}\int_{|\phi_{1,n}-\phi_1^*|>\delta_n^\prime} (g(a(\phi_1^*))-g(a(\phi_{1,n}))) \dd \mathbb{P} \notag\\
			\leq& \frac{D}{b_n} \biggl[\Bar{c}_6(\delta) \Bar{c}_1 \delta_n^{\prime 2} + (|g(a(\phi_1^*))| + |g(a(\phi_{1,n}))|) \mathbb{P}(|\phi_{1,n}-\phi_1^*|>\delta_n^{\prime})\biggl] \notag \\
			\leq& \frac{D}{b_n} \biggl[\Bar{c}_6(\delta) \Bar{c}_1 \delta_n^{\prime 2} + \frac{|g(a(\phi_1^*))| + \Bar{c}_4 \log n}{\delta_n^{\prime 2}} \E[|\phi_{1,n}-\phi_1^*|^2]\biggl] \notag\\
			\leq& \frac{D}{b_n} \biggl[\Bar{c}_6(\delta) \Bar{c}_1 \delta_n^{\prime 2} + \frac{|g(a(\phi_1^*))| + \Bar{c}_4 \log n}{\delta_n^{\prime 2}} \frac{\Bar{c}_7 (1\vee \log n)^{p} (1\vee \log\log n)^{\frac{4}{3}}}{n^{\frac{1}{2}}}\biggl] \notag\\
			=& \frac{2D}{b_n} \sqrt{\Bar{c}_7\Bar{c}_6(\delta) \Bar{c}_1 (|g(a(\phi_1^*))| + \Bar{c}_4 \log n)} \frac{(1\vee \log n)^{\frac{p}{2}} (1\vee \log\log n)^{\frac{2}{3}}}{n^{\frac{1}{4}}}. \notag
		\end{align}
		
		For $n<n_2$, by the same argument as in \eqref{eq_term_2} with $\delta_n^{\prime}$ replaced by $\delta$, we have
		\begin{equation}
			\begin{aligned}
				\label{eq_term_2_2}
				&\E[\phi_{2,n}(g(a(\phi_1^*))-g(a(\phi_{1,n})))] \\
				\leq& \frac{D}{b_n} \biggl[\Bar{c}_6(\delta) \Bar{c}_1 \delta^{2} + \frac{|g(a(\phi_1^*))| + \Bar{c}_4 \log n}{\delta^{2}} \frac{\Bar{c}_7 (1\vee \log n)^{p} (1\vee \log\log n)^{\frac{4}{3}}}{n^{\frac{1}{2}}}\biggl].
			\end{aligned}
		\end{equation}
		
		Finally, combining \eqref{eq_all_terms} -- \eqref{eq_term_2_2} yields

		\begin{align}
			&\sum_{n=1}^{N}\E[\Bar{J}(\phi_1^*, 0) - \Bar{J}(\phi_{1,n}, \phi_{2,n})] \notag\\
			=& \sum_{n=1}^{N}\E[-\phi_{2,n}g(a(\phi_1^*))] + \sum_{n=1}^{n_1-1}\E[f(a(\phi_1^*)) - f(a(\phi_{1,n}))] + \sum_{n=n_1}^{N}\E[f(a(\phi_1^*)) - f(a(\phi_{1,n}))] \notag  \\
			&+ \sum_{n=1}^{n_2-1}\E[\phi_{2,n}(g(a(\phi_1^*)) - g(a(\phi_{1,n}))] + \sum_{n=n_2}^{N}\E[\phi_{2,n}(g(a(\phi_1^*)) - g(a(\phi_{1,n}))] \notag \\
			\leq& \sum_{n=1}^{N}\frac{c}{n^{\frac{1}{4}}} + \sum_{n=1}^{n_1-1} \Bar{c}_5(\delta) \Bar{c}_1 \delta^2 + \frac{|f(a(\phi_1^*))| + \Bar{c}_{3} \log n}{\delta^2} \frac{\Bar{c}_7 (1\vee \log n)^{p} (1\vee \log\log n)^{\frac{4}{3}}}{n^{\frac{1}{2}}} \notag \\
			&+ 2\sum_{n=n_1}^{N} \sqrt{\Bar{c}_7\Bar{c}_5(\delta) \Bar{c}_1 (|f(a(\phi_1^*))| + \Bar{c}_{3} \log n)} \frac{(1\vee \log n)^{\frac{p}{2}} (1\vee \log\log n)^{\frac{2}{3}}}{n^{\frac{1}{4}}} \notag  \\
			&+ \sum_{n=1}^{n_2-1} \frac{D}{b_n} \biggl[\Bar{c}_6(\delta) \Bar{c}_1 \delta^{2} + \frac{|g(a(\phi_1^*))| + \Bar{c}_4 \log n}{\delta^{2}} \frac{\Bar{c}_7 (1\vee \log n)^{p} (1\vee \log\log n)^{\frac{4}{3}}}{n^{\frac{1}{2}}}\biggl] \notag \\
			&+ 2\sum_{n=n_2}^{N} \frac{D}{b_n} \sqrt{\Bar{c}_7\Bar{c}_6(\delta) \Bar{c}_1 (|g(a(\phi_1^*))| + \Bar{c}_4 \log n)} \frac{(1\vee \log n)^{\frac{p}{2}} (1\vee \log\log n)^{\frac{2}{3}}}{n^{\frac{1}{4}}} \notag \\
			\leq& c^\prime + c N^{\frac{3}{4}}(\log N)^{\frac{p+1}{2}} (\log\log N)^{\frac{2}{3}}. \notag 
		\end{align}

		The proof is complete.
	\end{proof}

	\begin{remark}
		We are able to specify the dependence of the constant \( p \) appearing in Theorems \ref{thm_convergence_and_rate_in_main} and \ref{thm_regret_in_main}  on the model parameters, and it turns out that $p$ does not depend on the temperature parameter \( \gamma \) and the control dimension \( l \). More precisely, going through careful (yet tedious) calculations, we establish
		\[
		p = \bar{c}(\bar{A}^2 + \bar{B}^2 + \bar{C}^2 + \bar{D}^2) T,
		\]		
		where
		\[
		\bar{A}=|A|\vee 1, \quad \bar{B}=|B|\vee 1, \quad \bar{C}=\sum_j |C_j|\vee 1, \quad \bar{D}=\sum_j |D_j|\vee 1,
		\]
		and \(\bar{c}>0\) is a universal constant independent of the model parameters. Details are left to interested readers.
	\end{remark}
	

	\section{Numerical Experiments}
	\label{section_numerical_experiment}

        \setcounter{equation}{0}
        \setcounter{theorem}{0}
        
        \setcounter{algorithm}{0}

	This section reports the results of numerically evaluating the convergence rate of \(\phi_{1,n}\) and the sublinear regret bound of our RL-LQ algorithm, compared with a benchmark algorithm. The benchmark adapts the model-based methods in \citep{basei2022logarithmic, szpruch2024optimal}  to our setting of state- and control-dependent volatility.

	\subsection{Simulation Setup}
	In our simulation study, we consider the case where \(l = 1\) for the control dimension and \(m = 1\) for the Brownian motion dimension. The controlled system \eqref{eq_classical_dynamics} simplifies to the following form:
	\begin{equation}
		\label{eq_classical_dynamics_one_dim}
		\mathrm{d}x^u(t) = (A x^u(t) + B u(t)) \mathrm{d}t + (C x^u(t) + D u(t)) \mathrm{d}W(t).
	\end{equation}
	
	In addition, we set the model parameters \(A, B, C, D, Q, H, x_0, T\) to be all 1, and set the exploration schedule \(b_n = 0.2 (n+1)^{1/4}\). Other sequences such as that of the learning rate \(\{a_{n}\}\) are configured according to the assumptions stated in Theorem \ref{thm_convergence_and_rate_in_main} and Subsection \ref{subsection_paramtrization}; for details see Appendix \ref{subsection_experiment_setup_app}. In each experiment we execute both our proposed Algorithm \ref{algo_rl-lq} and the benchmark Algorithm \ref{algo_model-based} over \(N = 200,000\) iterations,
	while we replicate the experiment independently for 120 times to draw statistical conclusions.
	
	\subsection{A Modified Model-Based Algorithm}
	The algorithms proposed in \citep{basei2022logarithmic, szpruch2024optimal} are designed to estimate parameters \(A\) and \(B\) in the drift term under the constant volatility setting. To compare with our algorithm tailored for state- and control-dependent volatilities, we extend their methods to include estimating the parameters \(C\) and \(D\). The details of this modified algorithm are described in Appendix \ref{subsection_modified_algo_app}.
	
	\subsection{Analysis of Numerical Results}We
	Figures \ref{figure_rate_ours} and \ref{figure_rate_plugin} compare the mean-squared convergence rates of \(\phi_{1,n}\) for our model-free Algorithm \ref{algo_rl-lq} and the model-based Algorithm \ref{algo_model-based}, using a log-log plot of Mean Squared Error (MSE) versus iterations.  The fitted linear regression shows our model's slope of \(-0.50\), confirming Theorem \ref{thm_convergence_and_rate_in_main} and outperforming the model-based benchmark slope of \(-0.09\).
	
	\begin{figure}[ht]
		\centering
		\begin{minipage}{0.49\textwidth}
			\includegraphics[width=\linewidth]{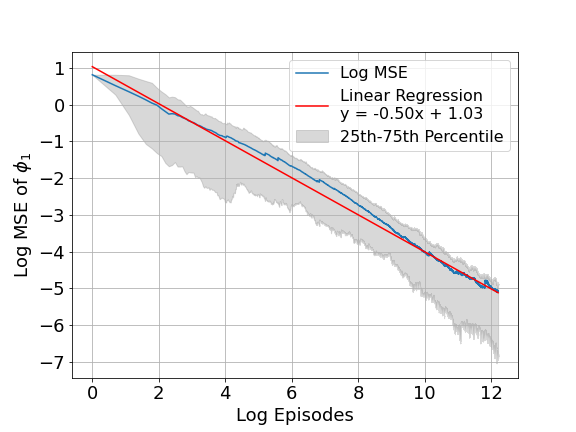}
			\caption{Log-Log plot of MSE of learned $\phi_{1,n}$ in Algorithm  \ref{algo_rl-lq}}
			\label{figure_rate_ours}
		\end{minipage}\hfill
		\begin{minipage}{0.49\textwidth}
			\includegraphics[width=\linewidth]{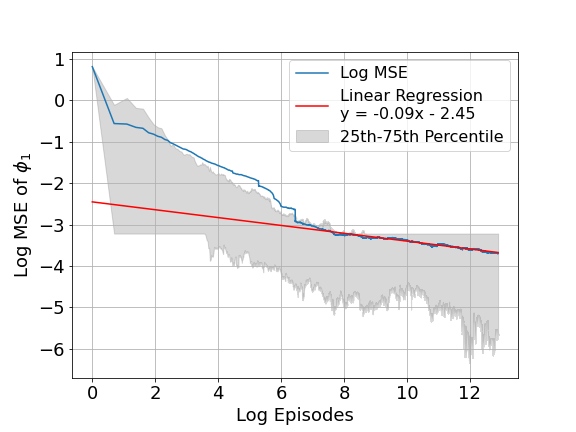}
			\caption{Log-Log plot of MSE of learned $\phi_{1,n}$ in the benchmark algorithm }
			\label{figure_rate_plugin}
		\end{minipage}
	\end{figure}
	
	A comparison of regrets between Algorithms \ref{algo_rl-lq} and  \ref{algo_model-based} is presented in Figures \ref{figure_regret_ours} and \ref{figure_regret_plugin}. The former yields a regret slope of around 0.73, which is close to the theoretical bound stipulated in Theorem \ref{thm_regret_in_main} and superior to the slope of 0.83 achieved by the latter. 
	
	\begin{figure}[ht]
		\centering
		\begin{minipage}{0.49\textwidth}
			\includegraphics[width=\linewidth]{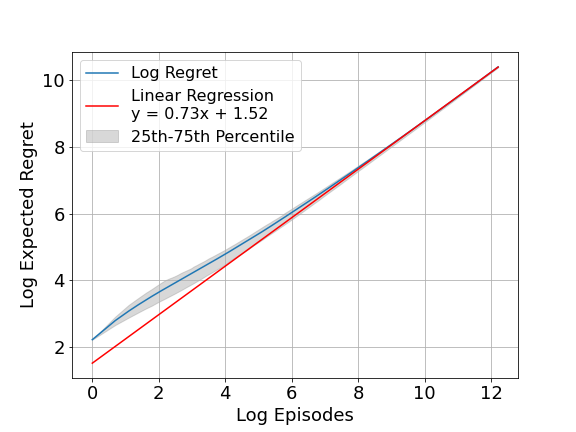}
			\caption{Log-Log plot of the expected regret of  Algorithm \ref{algo_rl-lq}}
			\label{figure_regret_ours}
		\end{minipage}\hfill
		\begin{minipage}{0.49\textwidth}
			\includegraphics[width=\linewidth]{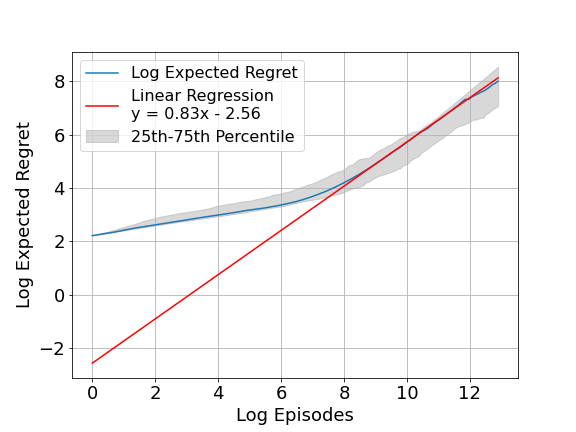}
			\caption{Log-Log plot of the expected regret of the benchmark algorithm}
			\label{figure_regret_plugin}
		\end{minipage}
	\end{figure}
	
	These experimental results support the theoretical claims and demonstrate the outperformance of our RL-LQ algorithm compared with its model-based counterpart in terms of both the convergence rates of the policy parameters and the regret bounds.\footnote{The GitHub repository of the experiments can be accessed at \url{https://github.com/yiliehuang/sublinear-regret-LQ-RL}.}

	\section{Conclusions}
	\label{section_conclusions}
	
	This paper is the first to derive a convergence rate and a regret bound within the model-free framework of continuous-time entropy-regularized RL for controlled diffusion processes initiated by \cite{wang2020reinforcement}. Here, by model-free, we mean that neither theory nor algorithm involves estimating model parameters. While it deals with the LQ case, it treats the case in which the diffusion term depends both on state and control, one that has not been studied in the RL literature to our best knowledge.

	In this paper, the temperature parameter $\gamma$ is fixed and the policy variance $\phi_2$ has a prescribed decaying schedule, both independent of data. These two parameters can be regarded as levels of exploration directly controlled by the critic and the actor respectively.
    A companion paper \cite{HZ25} considers the same problem but develops different RL algorithms where these two exploration mechanisms are both data driven. 
	
	There are several limitations in the setting or results of the paper. First, the state is one-dimensional and the quadratic objective functional (\ref{eq_classical_lq}) has no running reward from controls, which are key assumptions needed to simplify our analysis so that it suffices to consider only (time-invariant) stationary policies.
	While these assumptions are satisfied in some applications (e.g. in portfolio choice), imposing them is far from satisfactory. We hope that the present paper represents the {\it first step} towards solving the general LQ problem and some of the ideas here can inspire a more general convergence/regret analysis for continuous-time RL.
	Second, we are unable to achieve a better sublinear regret, e.g., a square-root one, which is typical in episodic RL algorithms for tabular or linear MDPs. We are  not certain whether that is due to our approach or it is more fundamental due to the diffusion nature of the system dynamics. Additionally, although the value function parameters \(\bm{\theta}\) do not require updates as they do not improve the regret bound, it will be interesting to explore whether updating them could yield better empirical results. Finally, extending the analysis to
	non-LQ problems with general function approximations is an enormous open question. All these point to exciting research opportunities in the (hopefully near) future.

	\bigskip
	
	\noindent{\bf Acknowledgement.} Huang and Zhou are supported by the Nie Center for Intelligent Asset Management at Columbia University. Their work is also part of a Columbia-CityU/HK collaborative project that is supported by the InnoHK Initiative, The Government of the HKSAR, and
	the AIFT Lab. Yanwei Jia is supported by the Start-up Fund and the Faculty Direct Fund 4055220 at The Chinese University of Hong Kong, and Hong Kong Research Grants Council (RGC) - Early Career Scheme (ECS) 2191379. We are especially indebted to the two anonymous referees for their constructive and detailed
	comments that have led to an improved version of the paper. In particular, one of the referees suggested the improved model-based algorithm in Appendix \ref{subsection_modified_algo_app}.

\newpage
\bibliography{references}

\newpage
\appendix

\renewcommand{\theequation}{\thesection.\arabic{equation}}
\renewcommand{\thetheorem}{\thesection.\arabic{theorem}}
\renewcommand{\thetheorem}{\thesection.\arabic{theorem}}
\renewcommand{\thelemma}{\thesection.\arabic{theorem}}  
\renewcommand{\theremark}{\thesection.\arabic{theorem}} 
\renewcommand{\thealgorithm}{\thesection.\arabic{algorithm}}

	\section{Specifics of Numerical Experiments}
	\label{section_details_numerical_experiments}

        \setcounter{equation}{0}
        \setcounter{theorem}{0}
        
        \setcounter{algorithm}{0}

	This section presents the implementation details of the numerical experiments outlined in Section \ref{section_numerical_experiment}. For clarity and simplicity, we set \(l = m = 1\) both our model-free continuous-time RL algorithm and the adapted model-based counterpart. To facilitate reproducibility, we fix the random seeds for all 120 independent experiments, ranging sequentially from 1 to 120. 
	
	The section is organized into three subsections: the first one presents and explains the modified algorithm that adapts the model-based methods \citep{basei2022logarithmic, szpruch2024optimal} to our setting involving state and control dependent volatility. The second one details the experimental conditions, parameter settings, and the overall framework used to validate our claims and assess the performance of our algorithmic enhancements. The third one describes the computational resources used for these experiments.

	\subsection{A Modified Model-Based Algorithm}
	\label{subsection_modified_algo_app}
	The key component in the algorithms developed by \citep{basei2022logarithmic, szpruch2024optimal} is to estimate the parameters \(A\) and \(B\) in the drift term whereas the diffusion term is assumed to be constant. Clearly, these algorithms cannot be used directly to our setting where the diffusion term is state- and control-dependent. Here we extend them to also including estimates of the parameters \(C\) and \(D\).

	Specifically, in the $n$-th iteration, a control policy is sampled from  the Gaussian policy
	\begin{equation}
		\label{eq_plugin_policy}
		u_n(t, x) \sim \mathcal{N}\left(\cdot\middle| \phi_{1,n} x, v_n\right),
	\end{equation}
	where $\phi_{1,n}=-\frac{B_n + C_nD_n}{D_n^2}$, $v_n = \frac{1}{n}$, and $B_n, C_n, D_n$ are the current estimates of the parameters $B, C, D$ respectively.
	Applying this  policy to the classical dynamic \eqref{eq_classical_dynamics} yields
	\begin{equation}
		\label{eq_plugin_dynamics}
		\begin{aligned}
			\dd x_n(t) = \left(A x_n(t) + B u_n(t)\right) d t + \left(C x_n(t) + D u_n(t)\right) d W_n(t),
		\end{aligned}
	\end{equation}
	where $W_n$ is the Brownian motion for the $n$-th iteration.
	
	Given an observed state trajectory $\{x_n(t) : 0 \leq t \leq T\}$ following \eqref{eq_plugin_dynamics}, we discretize it uniformly into $m$ intervals, resulting in the "snapshots" of the state, $\{x_n(t_0), x_n(t_1), \dots, x_n(t_m)\}$. Then we  estimate $A_n$, $B_n$, $C_n$, and $D_n$ after the $n$-th episode, utilizing all the observed data up to $n$.
	
	In particular, we estimate $A_n$ and $B_n$ using a least-square method as described in [\citealp[Equation 2.23]{basei2022logarithmic}]:
	\begin{equation}
		\begin{aligned}
			\label{eq_plugin_AB}
			\binom{A_n}{B_n} = & \left(\sum_{i=1}^n \sum_{k=1}^{m-1} \binom{x_i\left(t_k\right)}{u_i\left(t_k\right)} \binom{x_i\left(t_k\right)}{u_i\left(t_k\right)}^{\top} \Delta t + I \right)^{-1} \\
			& \times \left(\sum_{i=1}^n \sum_{k=1}^{m-1} \binom{x_i\left(t_k\right)}{u_i\left(t_k\right)} \left(x_i\left(t_{k+1}\right) - x_i\left(t_k\right)\right)\right).
		\end{aligned}
	\end{equation}
	On the other hand, 	noting the quadratic variation
	\begin{equation*}
		\left[x_i\right]_t = \int_0^t \left(C x_i(s) + D u_i(s)\right)^2 d s, \qquad \text{for } i = 1, 2, \dots, n,
	\end{equation*}
	we estimate $C_n$ and $D_n$ by minimizing the following loss function:
	\begin{equation}
		\label{eq_plugin_CD}
		\sum_{i=1}^n \left(\sum_{k=1}^{m-1} \left(x_i\left(t_k\right) - x_i\left(t_{k-1}\right)\right)^2 - \sum_{k=1}^{m-1} \left(C x_i\left(t_k\right) + D u_i\left(t_k\right)\right)^2 \Delta t \right)^2.
	\end{equation}
	
	The pseudocode for implementing this modified model-based algorithm is presented below in Algorithm \ref{algo_model-based}.
	
	\begin{algorithm}[htbp]
		\caption{Modified Model-Based Algorithm}\label{algo_model-based}
		\begin{algorithmic}
			\Input
			\Desc{$A_0, B_0$}{Initial drift parameters.}
			\Desc{$C_0, D_0$}{Initial diffusion parameters.}
			\EndInput
			\\\hrulefill
			\For{\texttt{$n = 1$ to $N$}}
			\State Initialize $k=0$, time $t_k = 0$, state $x_n(t_k) = x_0$.
			\For {\texttt{$k = 1$ to $m-1$}}
			\State Calculate the estimated policy mean parameter \(\phi_{1,n}=-\frac{B_n + C_nD_n}{D_n^2}\).
			\State Generate action $u_n(t_k, x) \sim \mathcal{N}\left(\cdot\middle| \phi_{1,n} x, v_n\right)$ following policy \eqref{eq_plugin_policy}.
			\State Apply action $u_n(t_k)$ and get new state $x_n(t_{k+1})$ by dynamic \eqref{eq_plugin_dynamics}.
			\State Update time $t_{k+1} \leftarrow t_k + \Delta t$.
			\EndFor
			\State Update $A_n, B_n$ using \eqref{eq_plugin_AB}.
			\State Update $C_n, D_n$ using \eqref{eq_plugin_CD}.
			\EndFor
			\\\hrulefill
			\Output
			\Desc{$A_N, B_N$}{Final estimated drift parameters.}
			\Desc{$C_N, D_N$}{Final estimated diffusion parameters.}
			\EndOutput
		\end{algorithmic}
	\end{algorithm}

	\subsection{Experiment Setup}
	\label{subsection_experiment_setup_app}
	The specific  setup for the experiment applying the model-free Algorithm \ref{algo_rl-lq}  is as follows:
	\begin{itemize}
        \setlength\itemsep{0.1em}
		\item The initial value for $\phi_1$ is $\phi_{1,0}=-0.5$.
		\item The leaning rate of $\phi_1$ is $a_{n} = \frac{0.05}{(n+1)^{\frac{3}{4}}}$.
		\item The projection for $\phi_{1,n}$ is set to be a constant set of $[-2.2, -0.5]$ for computational efficiency.\footnote{The projection was originally set to prove the theoretical convergence rate and regret bound. For implementation the theoretical projection grows too slow; instead it could be tuned. }
		\item The exploration schedule is $\phi_{2,n}=\frac{1}{b_n}$ where $b_n = 0.2 (n+1)^{\frac{1}{4}}$.
		\item The functions $\hat{k}_1(t;\bm\theta)=1$ and $\hat{k}_3(t;\bm\theta)=1$ for simplicity, which satisfy the assumptions in Subsection \ref{subsection_paramtrization}.\footnote{Recall that our results do not dependent on the form of the value function.}
		\item The parameters $\bm{\theta}$ need not to be learned, as the value function is not updated. 
		\item The temperature parameter $\gamma=1$.
		\item The initial state $x_0=1$.
		\item The time horizon $T=1$.
		\item The time step $\Delta t = 0.01$.
		\item The total number of episodes for each experiment $N=200,000$.
	\end{itemize}

	The specifics  of implementing the adapted model-based Algorithm \ref{algo_model-based} are:
	\begin{itemize}
        \setlength\itemsep{0.1em}
		\item The initial estimates of the model parameters are set to be $A = B = C = D = -2$, aligning with the initial value $\phi_{1,0} = -0.5$ in Algorithm \ref{algo_rl-lq}.
		\item The initial variance of the stochastic policy is set to be $v_0 = 5$, consistent with the initial value $\phi_{2,0} = 5$ in Algorithm \ref{algo_rl-lq}.
		\item To ensure fairness, we apply the same projection set $[-2.2, -0.5]$ to $\phi_{1,n}$ as in Algorithm \ref{algo_rl-lq} to stabilize the learning process.
		\item The initial state $x_0 = 1$.
		\item The time horizon $T = 1$.
		\item The time step $\Delta t = 0.01$.
		\item The total number of episodes for each experiment is $N = 200,000$.
	\end{itemize}
	
	For Figures \ref{figure_rate_ours} - \ref{figure_regret_plugin}, the regression line is fitted using data from the 5000th iteration onward to mitigate the impact of initial noise.

	\subsection{Compute Resources}
	\label{subsection_compute_resources}
	All experiments were performed on a MacBook Pro (16-inch, 2019) equipped with a 2.4 GHz 8-Core Intel Core i9 processor, 32 GB of 2667 MHz DDR4 memory, and dual graphics processors, comprising an AMD Radeon Pro 5500M with 8 GB and an Intel UHD Graphics 630 with 1536 MB. Not having a high-powered server, this consumer-grade laptop was sufficient to handle the computational task of conducting 120 independent experiments sequentially, each running for 200,000 episodes. The model-free actor--critic algorithm required approximately 10 hours for a complete sequential run, whereas the model-based plugin algorithm took about 20 hours. This significant difference in running times also demonstrates the efficiency of our model-free approach compared with the model-based one.

	\newpage
	\section{Additional Details about Proofs}
	\label{section_additional_proof_app}

        \setcounter{equation}{0}
        \setcounter{theorem}{0}
        
        \setcounter{algorithm}{0}

	\subsection{Moment Estimates}
	\label{section_moments_app}
	Let $\{x^{\bm\phi}(t): 0\leq t\leq T\}$ be the state process under the policy \eqref{policy_parametrization} following the dynamic \eqref{rl_dynamics}. The following lemma gives some moment estimates of  \(x^{\bm\phi}(t)\) in terms of \(\bm\phi=(\phi_1,\phi_2)\).
	
	\begin{lemma}
		\label{lemma_xt-w}
		We have
		\[
		\E[x^{\bm\phi}(t)] = x_0 e^{(A + B^\top \phi_1)t},
		\]
		\[
		\begin{aligned}
			\E [x^{\bm\phi}(t)^2] &= (\Sj  D_j^\top \phi_2 D_j) \int_0^t e^{a(\phi_1)(t-s)} \dd s + x_0^2e^{a(\phi_1)t},
		\end{aligned}
		\]
		where
		\begin{equation}
			\label{eq_a_phi_1}
			a(\phi_1) = 2A + 2B^\top \phi_1 + \Sj (C_j^2 + 2C_jD_j^\top\phi_1 + D_j^\top\phi_1\phi_1^\top D_j).
		\end{equation}
		Moreover, there exists a constant $c > 0$ (that only depends on $A, B, C, D, x_0$) such that
		\[
		\E[x^{\bm\phi}(t)^2] \leq c  (1 + |\phi_2| t) \exp{\{c |\phi_1|^2 t\}},
		\]
		\[
		\E[x^{\bm\phi}(t)^4] \leq c  (1 + |\phi_2|^2 t) \exp{\{c |\phi_1|^4 t\}},
		\]
		\[
		\E[x^{\bm\phi}(t)^6] \leq c  (1 + |\phi_2|^3 t) \exp{\{c |\phi_1|^6 t\}}.
		\]
	\end{lemma}
	
	\begin{proof}
		We have 
		\[
		\begin{aligned}
			x^{\bm\phi}(t) &= x(0) + \int_0^t \left(Ax^{\bm\phi}(s) + B^\top \phi_1 x^{\bm\phi}(s)\right) \dd s \\
			&+ \int_0^t \Sj \sqrt{C_j^2x^{\bm\phi}(s)^2 + 2C_jD_j^\top\phi_1 x^{\bm\phi}(s)^2 + D_j^\top(\phi_1 \phi_1^\top x^{\bm\phi}(s)^2 + \phi_2)D_j} \dd W^{(j)}(s).
		\end{aligned}
		\]
		Taking expectation on both sides, we have
		\[
		\E [x^{\bm\phi}(t)] = x_0 + \int_0^t (A \E[x^{\bm\phi}(s)] + B^\top \phi_1 \E[x^{\bm\phi}(s)] )\dd s,
		\]
		leading to
		\begin{equation}
			\label{x_moment1}
			\E[x^{\bm\phi}(t)] = x_0 e^{(A + B^\top \phi_1)t}.
		\end{equation}
		Next, applying Ito's formula to $x^{\bm\phi}(t)^2$ and taking expectation on both sides, we obtain
		\[
		\E[x^{\bm\phi}(t)^2] = x_0^2 + \int_0^t \left(a(\phi_1)\E[x^{\bm\phi}(s)^2] + \Sj D_j^\top \phi_2 D_j\right) \dd s,
		\]
		yielding
		\begin{equation}
			\label{x_moment2}
			\E[x^{\bm\phi}(t)^2] = (\Sj  D_j^\top \phi_2 D_j) \int_0^t e^{a(\phi_1)(t-s)} \dd s + x_0^2e^{a(\phi_1)t}.
		\end{equation}
		
		Now we prove the inequality related to $\E[x^{\bm\phi}(t)^6]$. Hölder's inequality yields
		\[
		\begin{aligned}
			&\E[x^{\bm\phi}(t)^6] \\
			=& \E\biggl[\biggl( x(0) + \int_0^t Ax^{\bm\phi}(s) + B^\top \phi_1 x^{\bm\phi}(s) \dd s \\
			&+ \int_0^t \Sj \sqrt{C_j^2x^{\bm\phi}(s)^2 + 2C_jD_j^\top\phi_1 x^{\bm\phi}(s)^2 + D_j^\top(\phi_1 \phi_1^\top x^{\bm\phi}(s)^2 + \phi_2)D_j} \dd W^{(j)}(s) \biggl)^6\biggl]\\
			\leq& c x_0^6 + c \E\biggl[\biggl(\int_0^t Ax^{\bm\phi}(s) + B^\top \phi_1 x^{\bm\phi}(s) \dd s\biggl)^6\biggl] + c\E\biggl[\biggl(\int_0^t \Sj \\
			&\qquad\sqrt{C_j^2x^{\bm\phi}(s)^2 + 2C_jD_j^\top\phi_1 x^{\bm\phi}(s)^2 + D_j^\top(\phi_1 \phi_1^\top x^{\bm\phi}(s)^2 + \phi_2)D_j} \dd W^{(j)}(s)\biggl)^6\biggl]\\
			\leq& c x_0^6 + c (A + B^\top \phi_1)^6 \E\biggl[\biggl(\int_0^t x^{\bm\phi}(s) \dd s\biggl)^6\biggl] \\
			&+ c\E\biggl[ \Sj \biggl(\int_0^t C_j^2x^{\bm\phi}(s)^2 + 2C_jD_j^\top\phi_1 x^{\bm\phi}(s)^2 + D_j^\top(\phi_1 \phi_1^\top x^{\bm\phi}(s)^2 + \phi_2)D_j) \dd s\biggl)^3\biggl]\\
			\leq& c x_0^6 + c(1+|\phi_1|^6) \E\biggl[\int_0^t x^{\bm\phi}(s)^6 \dd s\biggl] + c\E\biggl[ \int_0^t (1+|\phi_1|^6)x^{\bm\phi}(s)^6+|\phi_2|^3 \dd s \biggl].
		\end{aligned}
		\]
		It follows from Grönwall's inequality that
		\[
		\begin{aligned}
			\E[x^{\bm\phi}(t)^6]\leq& c(1+|\phi_2|^3t) \exp{\{\int_0^t c(1+|\phi_1|^6) \dd s\}}\\
			=& c(1+|\phi_2|^3t) \exp{ \{c(1+|\phi_1|^6) t\}}\\
			\leq&c(1+|\phi_2|^3t) \exp{ \{c|\phi_1|^6 t\}}.
		\end{aligned}
		\]
		The proofs for the inequalities of $\E[x^{\bm\phi}(t)^2]$ and $\E[x^{\bm\phi}(t)^4]$ are similar.
	\end{proof}

	The next lemma  concerns the variance of the increment \(Z_{1, n}(T)\).
	
	\begin{lemma}
		\label{lemma_nosie_app}
		There exists a constant \(c > 0\) that depends only on the model primitives such that
		
		\begin{equation}
			\label{eq:noise_upper1}
			\begin{aligned}
				\operatorname{Var}\left( Z_{1,n}(T) \Big| \bm\theta_n, \phi_{1,n}, \phi_{2,n} \right)  \leq& c b_n\left( 1 +  |\phi_{1,n}|^{8} + (\log b_n)^8\right) \exp{\{c|\phi_{1,n}|^6\}}.\\
			\end{aligned}
		\end{equation}
	\end{lemma}
	
	\begin{proof}
		Applying Ito's lemma to the process \(J\left(t, x_n(t); \bm{\theta}_n \right)\), where $x_n$ follows \eqref{eq_dynamics_app}, we have
		\[
		\begin{aligned}
			&\mathrm{d} J\left(t, x_n(t); \bm\theta_n \right) \\
			=& \biggl( -\frac{1}{2}k_1^\prime(t; \bm\theta_n)x_n(t)^2 + k_3^\prime(t; \bm\theta_n) - (Ax_n(t) + B^\top u_n(t))k_1(t; \bm\theta_n)x_n(t) \\
			&-\frac{\Sj (C_jx_n(t)+D_j^\top u_n(t))^2}{2} k_1(t; \bm\theta_n) \biggl) \dd t \\
			&- \Sj \biggl( (C_jx_n(t)+D_j^\top u_n(t)) k_1(t; \bm\theta_n) x_n(t) \biggl) \dd W_n^{(j)}(t).
		\end{aligned}
		\]
		In addition, 
		\[
		{p}(t, \bm\phi_n) = \frac{1}{2} \log( \det (\phi_{2,n})) + \frac{l}{2}\log(2 \pi e).
		\]
		Hence
		\begin{equation}
			\label{eq_dz1_app}
			\begin{aligned}
				&\dd Z_{1, n}(t) \\
				=& \phi_{2,n}^{-1}(u_n(t)-\phi_{1,n}x_n(t))x_n(t) \biggl [ \biggl( -\frac{1}{2}k_1^\prime(t; \bm\theta_n)x_n(t)^2 + k_3^\prime(t; \bm\theta_n)  \\
				&- (Ax_n(t) + B^\top u_n(t))k_1(t; \bm\theta_n)x_n(t) - \frac{\Sj (C_jx_n(t)+D_j^\top u_n(t))^2}{2} k_1(t; \bm\theta_n) \biggl) \dd t  \\
				&- \Sj \biggl( (C_jx_n(t)+D_j^\top u_n(t)) k_1(t; \bm\theta_n) x_n(t) \biggl) \dd W_n^{(j)}(t) - \frac{1}{2}Qx_n(t)^2 \dd t\\
				&+ \gamma \biggl( \frac{1}{2} \log( \det (\phi_{2,n})) + \frac{l}{2}\log(2 \pi e) \biggl) \dd t \biggl]\\
				=& \phi_{2,n}^{-1}(u_n(t)-\phi_{1,n}x_n(t))x_n(t) \biggl [ \biggl( -\frac{1}{2}k_1^\prime(t; \bm\theta_n)x_n(t)^2 + k_3^\prime(t; \bm\theta_n)  \\
				&- (Ax_n(t) + B^\top u_n(t))k_1(t; \bm\theta_n)x_n(t)\\
				&- \frac{\Sj (C_jx_n(t)+D_j^\top u_n(t))^2}{2} k_1(t; \bm\theta_n) \biggl) \\
				&- \frac{1}{2}Qx_n(t)^2 + \gamma \biggl( \frac{1}{2} \log( \det (\phi_{2,n})) + \frac{l}{2}\log(2 \pi e) \biggl) \biggl] \dd t\\
				&- \phi_{2,n}^{-1}(u_n(t)-\phi_{1,n}x_n(t))x_n(t)  \Sj \biggl( (C_jx_n(t)+D_j^\top u_n(t)) k_1(t; \bm\theta_n) x_n(t) \biggl) \dd W_n^{(j)}(t)\\
				\triangleq& Z_{1,n}^{(1)}(t) \dd t + \Sj Z_{1,n}^{(2, j)}(t) \dd W_n^{(j)}(t).
			\end{aligned}
		\end{equation}
		
		We now estimate 
		\[
		\begin{aligned}
			\E[|Z_{1,n}^{(1)}(t)|^2 | \bm\theta_n, \phi_{1,n}, \phi_{2,n}, x_n(t)] \leq& c  \biggl[ (1+|\phi_{1,n}|^4)|\phi_{2,n}^{-1}|x_n(t)^6 + (1+|\phi_{1,n}|^2)x_n(t)^4 \\
			&+ (1 + |\phi_{2,n}|^2+(\log(\det(\phi_{2,n})))^4)|\phi_{2,n}^{-1}|x_n(t)^2 \biggl],
		\end{aligned}
		\]
		and
		\[
		\begin{aligned}
			\E[|Z_{1,n}^{(2,j)}(t)|^2 | \bm\theta_n, \phi_{1,n}, \phi_{2,n}, x_n(t)] \leq& c \biggl[ 1 + |\phi_{2,n}^{-1}|(1+|\phi_{1,n}|^2)x_n(t)^6 +  x_n(t)^4\biggl].
		\end{aligned}
		\]
		By Lemma \ref{lemma_xt-w}, taking expectations in the above with respect to $x_n(t)$, we deduce
		\begin{equation}
			\label{eq_Z1_square_app}
			\begin{aligned}
				&\E[|Z_{1,n}^{(1)}(t)|^2 + \Sj |Z_{1,n}^{(2,j)}(t)|^2 | \bm\theta_n, \phi_{1,n}, \phi_{2,n}] \\
				\leq& c  \biggl[ (1+|\phi_{1,n}|^4)|\phi_{2,n}^{-1}|(1+|\phi_{2,n}t|^3)\exp{\{c|\phi_{1,n}|^6t\}} \\
				&+(1+|\phi_{1,n}|^2)(1+|\phi_{2,n}t|^2)\exp{\{c|\phi_{1,n}|^4t\}} \\
				&+ (1+|\phi_{2,n}|^2+(\log(\det (\phi_{2,n})))^4)|\phi_{2,n}^{-1}|(1+|\phi_{2,n}t|)\exp{\{c|\phi_{1,n}|^2t\}} \biggl]. \\
			\end{aligned}
		\end{equation}
		Recalling that $\phi_{2,n}=\frac{I_l}{b_n}$ set  in Algorithm \ref{algo_rl-lq}, we arrive at \eqref{eq:noise_upper1}.
	\end{proof}

	\subsection{Mean Increment}
	\label{appendix:mean increment}	
	We now analyze the mean increment \\
	\(h_1(\phi_{1,n}, \phi_{2,n}; \bm\theta_n)\) in the updating rule \eqref{eq_phi1_update_app}. 
	First, note that \(\int_0^{\cdot} Z_{1,n}^{(2,j)}(t) \dd W_n^{(j)}(t)\) is a martingale by virtue of Lemma \ref{lemma_xt-w} and \eqref{eq_Z1_square_app}.
	Taking  integration and expectation in \eqref{eq_dz1_app}, we get
	\begin{equation}
		\label{eq_dez1_app}
		\begin{aligned}
			\E[Z_{1,n}(s)] 
			&= -\int_0^s k_1(t; \bm\theta_n)(B+\Vj+(\Sj D_j D_j^\top)\phi_{1,n}) \E [x_n(t)^2] \dd t,
		\end{aligned}
	\end{equation}
	where $0\leq s\leq T$.
	Hence
	\begin{equation}
		\label{eq_h1_app}
		\begin{aligned}
			h_1(\phi_{1,n}, \phi_{2,n}; \bm\theta_n) &= -(B+\Vj+(\Sj D_j D_j^\top)\phi_{1,n})\int_0^T k_1(t; \bm\theta_n)\E [x_n(t)^2]  \dd t\\
			&=-l(\phi_{1,n}, \phi_{2,n}; \bm\theta_n)(\phi_{1,n} - \phi_1^*),
		\end{aligned}
	\end{equation}
	where
	\begin{equation}
		\label{eq_l_definition}
		l(\phi_{1,n}, \phi_{2,n}; \bm\theta_n) = (\Sj D_j D_j^\top) \int_0^T k_1(t; \bm\theta_n)\E [x_n(t)^2] \dd t.
	\end{equation}

	Next we study $h_1(\phi_{1,n}, \phi_{2,n}; \bm\theta_n)$. It follows from  \eqref{eq_a_phi_1} that $a(\phi_1)$ is a quadratic function of $\phi_1$ and
	\[a(\phi_1) \geq 2A + \Sj C_j^2 - (B+\Vj)^\top \Mj (B+\Vj).\]
	Hence, by
	\eqref{x_moment2}, we have 
	\begin{equation}
		\label{eq_m2_lower}
		\begin{aligned}
			\E [x_n(t)^2] &= (\Sj  D_j^\top \phi_{2,n} D_j) \int_0^t e^{a(\phi_{1,n})(t-s)} \dd s + x_0^2e^{a(\phi_{1,n})t}\\
			&\geq x_0^2e^{a(\phi_{1,n})t} \geq c,
		\end{aligned}
	\end{equation}
	where $c>0$ is a constant independent of $n$.
	Thus,
	\begin{equation}
		\label{eq_h1_l_lower_bound_app}
		\begin{aligned}
			l(\phi_{1,n}, \phi_{2,n}; \bm\theta_n) &= (\Sj D_j D_j^\top) \int_0^T k_1(t; \bm\theta_n)\E [x_n(t)^2] \dd t\\
			&\geq (\Sj D_j D_j^\top) \int_0^T (1/c_2)c  \dd t\\
			&= (\Sj D_j D_j^\top) (1/c_2)c T \geq \bar{c} I_l,
		\end{aligned}
	\end{equation}
	where $0 < \bar{c} < 1$ is a constant independent of $n$.
	
	
	
	On the other hand, since  $a(\phi_1)$ is a quadratic function of $\phi_1$, we have
	\[ |a(\phi_1)|\leq c(1+|\phi_1|^2),\;\;\forall \phi_1\in \mathbb{R}^l,\]
	for some constant $c>0$. So
	\[
	\begin{aligned}
		\E [x_n(t)^2] &= (\Sj  D_j^\top \phi_{2,n} D_j) \int_0^t e^{a(\phi_{1,n})(t-s)} \dd s + x_0^2e^{a(\phi_{1,n})t}\\
		&\leq c(1+|\phi_{2,n}|)Te^{c|\phi_{1,n}|^2T},
	\end{aligned}
	\]
	leading to
	\[
	\begin{aligned}
		l(\phi_{1,n}, \phi_{2,n}; \bm\theta_n) &= (\Sj D_j D_j^\top) \int_0^T k_1(t; \bm\theta_n)\E [x_n(t)^2] \dd t\\
		&\leq (\Sj D_j D_j^\top) c_2T^2c(1+|\phi_{2,n}|)e^{c|\phi_{1,n}|^2T}.
	\end{aligned}
	\]
	Recalling that $\phi_{2,n} = \frac{I_l}{b_n}$, we arrive at
	\begin{equation}
		\begin{aligned}
			\label{eq_h1n_upper_app}
			|h_1(\phi_{1,n}, \phi_{2,n}; \bm\theta_n)|& \leq |\phi_{1,n} - \phi_1^*| |l(\phi_{1,n}, \phi_{2,n}; \bm\theta_n)|\\
			&\leq c(1+|\phi_{1,n}|)e^{c|\phi_{1,n}|^2}
		\end{aligned}
	\end{equation}
	for some constant $c>0$.

	\subsection{Almost Sure Convergence of $\phi_{1,n}$}
	
	We are now ready to  prove Part (a) of Theorem \ref{thm_convergence_and_rate_in_main} regarding the almost sure convergence of \( \phi_{1,n} \). Here and henceforth we will prove more general results by allowing a bias term \(\beta_{1, n}=\E\left[ \xi_{1, n} \mid \mathcal{G}_n\right] \) that may account for errors arising from  practical implementations (e.g. the discretization errors; see the proof of Theorem \ref{thm_convergence_and_rate_in_main} for details). The following theorem specializes to Part (a) of Theorem \ref{thm_convergence_and_rate_in_main} when \(\beta_{1, n} = \mathbf{0}\).
	
	\begin{theorem}
		\label{thm:phi1_convergence}
		Assume the noise term $\xi_{1,n}$ satisfies $\E\left[ \xi_{1, n} \Big|\g_n\right] = \beta_{1, n}$ and
		\begin{equation}\label{eq:phi1_noise_assumption}
			\begin{aligned}
				\E\left[ \left|\xi_{1, n} - \beta_{1, n} \right|^2 \Big| \g_n \right]  \leq& c  b_n( 1 +  |\phi_{1,n}|^{8} + (\log b_n)^8) \exp{\{c|\phi_{1,n}|^6\}}, \\
			\end{aligned}
		\end{equation}
		where $c>0$ is a constant independent of $n$, and $\{\g_n\}$ is the filtration generated by $\{\bm\theta_m, \phi_{1,m}, \phi_{2,m}, m=0,1,2,...,n\}$.
		Moreover, assume
		\begin{equation} \label{eq_phi1_assumption_app}
			\begin{aligned}
				(i)& \quad 0 < a_n \leq 1 \text{ for all } n, \quad a_n \downarrow 0, \quad \sum_n a_n = \infty, \quad \sum_n a_n |\beta_{1, n}| < \infty; \\
				(ii)& \quad c_{1, n}\uparrow \infty, \quad \sum_n a_n^2 b_n^{q_1} (\log b_n)^{q_2} c_{1,n}^{q_3} e^{cc_{1, n}^{q_4}} < \infty \\
				&\quad \text{for any } c>0 \text{, } 0\leq q_1 \leq 1 \text{ , } 0\leq q_2 \leq 8 \text{ , } 0\leq q_3 \leq 8 \text{ and } 0\leq q_4 \leq 6; \\
				(iii)& \quad b_n\geq1 \text{ for all }n, \quad b_n\uparrow \infty.
			\end{aligned}
		\end{equation}
		Then $\phi_{1,n}$ almost surely converges to the unique equilibrium point
		\[\phi_1^* = -(\sum_{j=1}^{m} D_j D_j^\top)^{-1}(B + \sum_{j=1}^{m}C_jD_j).\]
	\end{theorem}

	\begin{proof}
		The main idea is to derive inductive upper bounds of $|\phi_{1,n} - \phi_1^*|^2$, namely, we will bound $|\phi_{1,n+1} - \phi_1^*|^2$ in terms of $|\phi_{1,n} - \phi_1^*|^2$.
		
		First, for any closed, convex set $K\subset \mathbb{R}$ and $x\in K, y\in \mathbb{R}$, it follows from a property of projection that the function $f(t)=|t\Pi_K(y) + (1-t) x - y|^2$, $t\in\mathbb{R}$,  achieves minimum at $t=1$. 
		The first-order condition at $t = 1$ then yields
		\[ 2 |\Pi_K(y) - y|^2 - 2\langle \Pi_K(y) - y, x - y \rangle = 0 .\]
		Therefore,
		\[
		\begin{aligned}
			|\Pi_K(y) - x|^2 =& |\Pi_K(y) - y + y - x |^2 = |y-x|^2 + | \Pi_K(y) - y|^2 \\
			&+ 2\langle  \Pi_K(y) - y, y - x \rangle \\
			=& |y - x|^2 - |\Pi_K(y) - y|^2 \leq |y-x|^2 .
		\end{aligned}
		\]
		
		Taking $n$ sufficiently large such that $\phi_1^*\in K_{1, n+1}$, we have
		\[ |\phi_{1, n+1} - \phi_1^*|^2 \leq |\phi_{1, n} + a_n[ h_1(\phi_{1,n}, \phi_{2,n}; \bm\theta_n) + \xi_{1, n}] - \phi_1^*|^2.\]
		
		Denoting $U_{1,n} = \phi_{1,n} - \phi_1^*$, we deduce
		\[\begin{aligned}
			& \E\left[|U_{1,n+1}|^2 \Big| \g_n \right] \\
			\leq & \E\left[| U_{1,n} +  a_n[ h_1(\phi_{1,n}, \phi_{2,n}; \bm\theta_n) + \xi_{1,n}] |^2 \Big| \g_n \right] \\
			= & |U_{1,n}|^2 + 2a_n \langle U_{1,n},  h_1(\phi_{1,n}, \phi_{2,n}; \bm\theta_n) +  \beta_{1,n} \rangle + a_n^2 \E\left[ | h_1(\phi_{1,n}, \phi_{2,n}; \bm\theta_n) + \xi_{1,n}|^2  \Big| \g_n \right] \\
			= & |U_{1,n}|^2 + 2a_n \langle U_{1,n},  h_1(\phi_{1,n}, \phi_{2,n}; \bm\theta_n) +  \beta_{1,n} \rangle \\
			&+ a_n^2 \E\left[ | h_1(\phi_{1,n}, \phi_{2,n}; \bm\theta_n) + (\xi_{1,n} -  \beta_{1,n}) +  \beta_{1,n}|^2  \Big| \g_n \right] \\
			\leq & |U_{1,n}|^2 + 2a_n \langle U_{1,n},  h_1(\phi_{1,n}, \phi_{2,n}; \bm\theta_n) \rangle + 2 a_n | \beta_{1,n}| |U_{1,n}| \\
			&+ 3a_n^2 \left( | h_1(\phi_{1,n}, \phi_{2,n}; \bm\theta_n)|^2 + | \beta_{1,n}|^2 +  \E\left[ \left|\xi_{1,n} -  \beta_{1,n} \right|^2 \Big| \g_n\right] \right)  \\
			\leq & |U_{1,n}|^2 + 2a_n \langle U_{1,n},  h_1(\phi_{1,n}, \phi_{2,n}; \bm\theta_n) \rangle + a_n | \beta_{1,n}| (1 +|U_{1,n}|^2) \\
			&+ 3a_n^2 \left( | h_1(\phi_{1,n}, \phi_{2,n}; \bm\theta_n)|^2 + | \beta_{1,n}|^2 +  \E\left[ \left|\xi_{1,n} -  \beta_{1,n} \right|^2 \Big| \g_n\right] \right) .
		\end{aligned}\]
		
		Recall that $|\phi_{1,n}| \leq c_{1, n}$ almost surely. By \eqref{eq_h1n_upper_app}, we have
		\[
		|h_1(\phi_{1,n}, \phi_{2,n}; \bm\theta_n)|^2 \leq c (1+|\phi_{1,n}|)^2e^{2c|\phi_{1,n}|^2} \leq c(1+c_{1,n}^2)e^{cc_{1,n}^2}.
		\]
		
		In addition, the assumption \eqref{eq:phi1_noise_assumption} yields
		\[
		\begin{aligned}
			\E\left[ \left|\xi_{1, n} - \beta_{1, n} \right|^2 \Big| \g_n \right]  \leq& c  b_n( 1 +  |\phi_{1,n}|^{8} + (\log b_n)^8) \exp{\{c|\phi_{1,n}|^6\}}\\
			\leq &c  b_n( 1 +  c_{1,n}^{8} + (\log b_n)^8) \exp{\{cc_{1,n}^6\}}.
		\end{aligned}
		\]
		Therefore,
		\[\begin{aligned}
			& \E\left[|U_{1,n+1}|^2 \Big| \g_n \right] \\
			\leq & |U_{1,n}|^2 + 2a_n \langle U_{1,n},  h_1(\phi_{1,n}, \phi_{2,n}; \bm\theta_n) \rangle + a_n | \beta_{1,n}| (1 +|U_{1,n}|^2) \\
			&+ 3a_n^2 \left( | h_1(\phi_{1,n}, \phi_{2,n}; \bm\theta_n)|^2 + | \beta_{1,n}|^2 +  \E\left[ \left|\xi_{1,n} -  \beta_{1,n} \right|^2 \Big| \g_n\right] \right) \\
			\leq & (1 + a_n |\beta_{1,n}|)|U_{1,n}|^2 + 2a_n \langle U_{1,n},  h_1(\phi_{1,n}, \phi_{2,n}; \bm\theta_n) \rangle + a_n | \beta_{1,n}|\\
			&+3a_n^2\bigg( c(1+c_{1,n}^2)e^{cc_{1,n}^2} + | \beta_{1,n}|^2 + c  b_n( 1 +  c_{1,n}^{8} + (\log b_n)^8) \exp{\{cc_{1,n}^6\}}\bigg)\\
			=&: (1 + \kappa_{1,n}) |U_{1,n}|^2  - \zeta_{1,n} + \eta_{1,n},
		\end{aligned}\]
		where $\kappa_{1,n} = a_n |\beta_{1,n}|$, $\zeta_{1,n} = -2a_n \langle U_{1,n}, h_1(\phi_{1,n}, \phi_{2,n}; \bm\theta_n) \rangle $, and
		\begin{equation}
			\label{eq:eta_1_n}
			\begin{aligned}
				\eta_{1,n} = a_n|\beta_{1,n}| + 3a_n^2\bigg(& c(1+c_{1,n}^2)e^{cc_{1,n}^2} + | \beta_{1,n}|^2 \\
				&+ c  b_n( 1 +  c_{1,n}^{8} + (\log b_n)^8) \exp{\{cc_{1,n}^6\}}\bigg).
			\end{aligned}
		\end{equation}
		
		By assumptions (i)-(ii) of \eqref{eq_phi1_assumption_app}, we know $\sum \kappa_{1,n}<\infty$ and $\sum \eta_{1,n}<\infty$. It then follows from \citep[Theorem 1]{robbins1971convergence} that $\left|U_{1,n}\right|^2$ converges to a finite limit and $\sum \zeta_{1,n}<\infty$ almost surely.
		
		It remains to show $|U_{1,n}| \to 0$ almost surely. It follows from \eqref{eq_h1_app} and \eqref{eq_h1_l_lower_bound_app} that
		\[
		\begin{aligned}
			\zeta_{1,n}&=-2a_n \langle U_{1,n}, h_1(\phi_{1,n}, \phi_{2,n}; \bm\theta_n) \rangle\\
			&=2a_n \langle (\phi_{1,n}-\phi_1^*), l(\phi_{1,n}, \phi_{2,n}; \bm\theta_n) (\phi_{1,n}-\phi_1^*) \rangle\\
			&= 2a_n (\phi_{1,n}-\phi_1^*)^\top l(\phi_{1,n}, \phi_{2,n}; \bm\theta_n) (\phi_{1,n}-\phi_1^*) \\
			&\geq 2\bar{c} a_n |\phi_{1,n}-\phi_1^*|^2.
		\end{aligned}
		\]
		
		Now, suppose $|U_{1,n}|^2 \rightarrow c$ almost surely, where $c$ is not almost surely 0. Then, there exists a measurable set $Z \in \mathcal{F}$ with $\mathbb{P}(Z) >0$ such that, for every $\omega \in Z$, there exists a constant $0<\delta(\omega)<c(\omega)$ satisfying
		\[
		|U_{1,n}(\omega)|^2 = |\phi_{1,n}(\omega) - \phi_1^*|^2 \geq c(\omega)-\delta(\omega) > 0 \quad \text{for sufficiently large } n.
		\]
		Thus,
		\[
		\sum \zeta_{1,n}(\omega) \geq \sum 2\bar{c} a_n |\phi_{1,n}(\omega)-\phi_1^*|^2 \geq 2\bar{c}(c(\omega)-\delta(\omega))\sum  a_n =\infty.
		\]
		This is a contradiction.
		Therefore, we have proved that $|U_{1,n}|$ converges to 0 almost surely, or $\phi_{1,n}$ converges to $\phi_1^*$ almost surely.
	\end{proof}
	
	
	\begin{remark}
		\label{remark_example_1}
		An instance satisfying the assumptions in \eqref{eq_phi1_assumption_app} is $a_n =1\wedge\frac{\alpha^{\frac{3}{4}}}{(n+\beta)^{\frac{3}{4}}}$, where constants $\alpha>0$, $\beta>0$. $b_n = 1\vee\frac{(n+\beta)^{\frac{1}{4}}}{\alpha^{\frac{1}{4}}}, c_{1, n} = 1\vee (\log \log n)^{\frac{1}{6}}, $ and $\beta_{1,n}=0$. This is because $\sum \frac{1}{n} = \infty$, and $\sum \frac{(\log n)^{p} (\log \log n)^{q}}{n^{r}} <\infty$, for any $p, q  > 0$ and $r > 1$.
	\end{remark}
	
	\subsection{Mean-Squared Error of $\phi_{1,n} - \phi_1^*$}
	\label{section_MSE_app}
	
	In this section, we establish the convergence rate of $\phi_{1,n}$ to $ \phi_1^*$ stipulated in part (b) of Theorem \ref{thm_convergence_and_rate_in_main}. 
	
	The following lemma shows a general recursive relation satisfied by some sequences of learning rates.
	\begin{lemma}
		\label{lemma_an_app}
		For any $w> 0$, there exists positive numbers $\alpha > \frac{1}{w}$ and $\beta \geq \max(\frac{1}{w \alpha -1}, w^2\alpha^3)$ such that the sequence $a_n=\frac{\alpha^{\frac{3}{4}}}{(n+\beta)^{\frac{3}{4}}}$ satisfies $a_n \leq a_{n+1}(1 + w a_{n+1})$ for all $n \geq 0$.
	\end{lemma}
	
	\begin{proof}
		We have the following equivalences:
		\begin{equation}
			\label{eq_lemma_an}
			\begin{aligned}
				&a_n \leq a_{n+1}(1 + w a_{n+1})\\
				\Leftrightarrow &\frac{\alpha^{\frac{3}{4}}}{(n + \beta)^{\frac{3}{4}}} \leq \frac{\alpha^{\frac{3}{4}}}{(n + 1 + \beta)^{\frac{3}{4}}} + w \left(\frac{\alpha}{n + 1 + \beta}\right)^{\frac{2}{3}}\\
				\Leftrightarrow & (n+\beta+1)^{\frac{3}{4}} - (n+\beta)^{\frac{3}{4}} \leq w \alpha^{\frac{3}{4}} \biggl( \frac{n+\beta}{n+\beta+1} \biggl)^{\frac{3}{4}}.
			\end{aligned}
		\end{equation}
		Consider the last inequality in \eqref{eq_lemma_an} and notice that the left hand side is a decreasing function of $n$ and the right hand side is an increasing function of $n$. So to show that this inequality is true for all $n$, it is sufficient to show that it is true when $n=0$, which is
		\begin{equation}\label{desired}
			(\beta+1)^{\frac{3}{4}} - \beta^{\frac{3}{4}} \leq w\alpha^{\frac{3}{4}} \frac{\beta^{\frac{3}{4}}}{(\beta+1)^{\frac{3}{4}}}.
		\end{equation}
		
		To this end, it follows from $\beta\geq w^2\alpha^3$ and $w\alpha\beta\geq \beta+1$ that
		\[
		w \beta^4\geq w^3\alpha^3\beta^3\geq (\beta+1)^3.
		\]
		Hence
		\[
		w^{\frac{1}{4}} \frac{\beta^{\frac{3}{4}}}{(\beta+1)^{\frac{3}{4}}}
		\geq \frac{3}{4}\beta^{-\frac{1}{4}} \geq
		(\beta+1)^{\frac{3}{4}} - \beta^{\frac{3}{4}},
		\]
		where the last inequality is due to the mean value theorem.
		Now the desired inequality \eqref{desired} follows from the fact that $\alpha > \frac{1}{w}$.
	\end{proof}
	
	The following result covers part (b) of Theorem \ref{thm_convergence_and_rate_in_main} as a special case. 
	
	\begin{theorem}
		\label{thm_phi_1_rate}
		Under the assumption of Theorem \ref{thm:phi1_convergence}, if the sequence $\{a_n\}$ further satisfies
		\[a_n \leq a_{n+1}(1 + w a_{n+1})\] for some sufficiently small constant $w>0$  and $\{\frac{b_n}{a_n}|\beta_{1,n}|^2\}$ is non-decreasing in $n$,
		then there exists an increasing sequence $\{\hat{\eta}_{1,n}\}$ 
		and a constant ${c}'>0$ such that
		\[
		\E [|\phi_{1, n+1} - \phi_1^*|^2] \leq {c}' a_n \hat{\eta}_{1,n}.
		\]
		In particular, if we set the parameters $a_n, b_n, c_{1,n}, \beta_{1,n}$ as in Remark \ref{remark_example_1}, then
		\[
		\E [|\phi_{1, n+1} - \phi_1^*|^2] \leq c \frac{(1 \vee \log n)^{p} (1 \vee \log \log n)^{\frac{4}{3}}}{n^{\frac{1}{2}}}
		\]
		for any $n$, where $c$ and $p$ are positive constants that only depend on model primitives.
	\end{theorem}
	
	\begin{proof}
		Denote $n_0=\inf \{n\geq0:\phi_1^*\in K_{1, n+1}\}$ and $U_{1, n} = \phi_{1, n} - \phi_1^*$. It follows from  \eqref{eq_h1_app} and \eqref{eq_h1_l_lower_bound_app} that
		\begin{equation}
			\label{eq:inner_product_upper}
			\langle U_{1,n}, h_1(\phi_{1,n}, \phi_{2,n}; \bm\theta_n)\rangle = - U_{1,n}^\top l(\phi_{1,n}, \phi_{2,n}; \bm\theta_n) U_{1,n}  \leq  -\bar{c} |U_{1,n}|^2.
		\end{equation}
		When $n \geq n_0$, this together with a similar argument to the proof of Theorem \ref{thm:phi1_convergence} yields
		\begin{equation}
			\label{eq:phi1_rate(t)emp}
			\begin{aligned}
				& \E\left[|U_{1,n+1}|^2 \Big| \g_n \right] \\
				\leq & |U_{1,n}|^2 + 2a_n \langle U_{1,n}, h_1(\phi_{1,n}, \phi_{2,n}; \bm\theta_n)\rangle  + 2 a_n | \beta_{1,n}| |U_{1,n}| \\
				&+ 3a_n^2 \left( | h_1(\phi_{1,n}, \phi_{2,n}; \bm\theta_n)|^2 + | \beta_{1,n}|^2 +  \E\left[ \left|\xi_{1,n} -  \beta_{1,n} \right|^2 \Big| \g_n\right] \right)  \\
				\leq & |U_{1,n}|^2 - 2a_n \bar{c} |U_{1,n}|^2 + 2 a_n | \beta_{1,n}| |U_{1,n}| \\
				&+ 3a_n^2 \left( | h_1(\phi_{1,n}, \phi_{2,n}; \bm\theta_n)|^2 + | \beta_{1,n}|^2 +  \E\left[ \left|\xi_{1,n} -  \beta_{1,n} \right|^2 \Big| \g_n\right] \right)  \\
				\leq & |U_{1,n}|^2 - 2a_n \bar{c} |U_{1,n}|^2 + a_n (\frac{1}{\bar{c}}|\beta_{1,n}|^2 + \bar{c}|U_{1,n}|^2) \\
				&+ 3a_n^2 \left( | h_1(\phi_{1,n}, \phi_{2,n}; \bm\theta_n)|^2 + | \beta_{1,n}|^2 +  \E\left[ \left|\xi_{1,n} -  \beta_{1,n} \right|^2 \Big| \g_n\right] \right)  \\
				=&(1-\bar{c} a_n )|U_{1,n}|^2 + 3a_n^2\bigg( | h_1(\phi_{1,n}, \phi_{2,n}; \bm\theta_n)|^2 + (1 + \frac{1}{3\bar{c} a_n})|\beta_{1,n}|^2 \\
				&+  \E\left[ \left|\xi_{1,n} -  \beta_{1,n} \right|^2 \Big| \g_n\right] \bigg).
			\end{aligned}
		\end{equation}
		Now, by the proof of Theorem \ref{thm:phi1_convergence},
		\[
		\begin{aligned}
			&| h_1(\phi_{1,n}, \phi_{2,n}; \bm\theta_n)|^2 +  \E\left[ \left|\xi_{1,n} -  \beta_{1,n} \right|^2 \Big| \g_n\right]\\
			\leq& c \biggl( (1+c_{1,n}^2)e^{cc_{1,n}^2} +  b_n( 1 +  c_{1,n}^{8} + (\log b_n)^8) \exp{\{cc_{1,n}^6\}} \biggl)\\
			\leq& c  b_n ( 1 +  c_{1,n}^{8} + (\log b_n)^8) \exp{\{cc_{1,n}^6\}}.
		\end{aligned}
		\]
		Moreover, the assumptions in \eqref{eq_phi1_assumption_app} imply that  $(1 + \frac{1}{3\bar{c} a_n})|\beta_{1,n}|^2 \leq c \frac{b_n}{a_n}|\beta_{1,n}|^2$ for some constant $c>0$. 
		When $n \geq n_0$, it follows from \eqref{eq:phi1_rate(t)emp} that
		\[ \E\left[|U_{1,n+1}|^2 \Big| \g_n \right] \leq (1-\bar{c} a_n)|U_{1,n}|^2 + 3a_n^2 \hat{\eta}_{1,n}, \]
		where
		\begin{equation}
			\label{eq_eta_hat_app}
			\begin{aligned}
				\hat{\eta}_{1,n}=& c b_n ( 1 + c_{1,n}^{8} + (\log b_n)^8 + \frac{|\beta_{1,n}|^2}{a_n} ) \exp{\{cc_{1,n}^6\}},
			\end{aligned}
		\end{equation}
		which is monotonically increasing because $c_{1,n}$, $b_n$ are monotonically increasing and $\frac{b_n}{a_n}|\beta_{1,n}|^2$ is non-decreasing by the assumptions. Taking expectation on both sides of the above and
		denoting $\rho_n = \E[|U_{1, n}|^2]$, we get
		\begin{equation}
			\label{eq:phi1_rate_induction}
			\rho_{n+1} \leq (1-\bar{c} a_n)\rho_n + 3a_n^2 \hat{\eta}_{1,n}
		\end{equation}
		when $n\geq n_0$.
		
		Next, we show $\rho_{n+1} \leq c'a_n\hat{\eta}_{1,n}$ for all $n \geq 0$ by induction, where\\
		$c' = max\{ \frac{\rho_1}{a_0 \hat{\eta}_{1,0}},  \frac{\rho_2}{a_1 \hat{\eta}_{1,1}}, \cdots,  \frac{\rho_{n_0+1}}{a_{n_0} \hat{\eta}_{1,{n_0}}}, \frac{3}{\bar{c}} \} + 1$.
		Indeed, it is true when $n\leq n_0$.  Assume that $\rho_{k+1} \leq c'a_k \hat{\eta}_{1,k}$ is true for $n_0 \leq k \leq n-1$. 
		Then \eqref{eq:phi1_rate_induction} yields
		\[
		\begin{aligned}
			\rho_{n+1} &\leq (1-\bar{c} a_n)\rho_n + 3a_n^2 \hat{\eta}_{1,n}\\
			&\leq (1-\bar{c}a_n)c'a_{n-1} \hat{\eta}_{1,n-1} + 3a_n^2 \hat{\eta}_{1,n} \\
			&\leq (1-\bar{c} a_n)c'a_{n}(1 + w a_n) \hat{\eta}_{1,n} + 3a_n^2 \hat{\eta}_{1,n} \\
			&=c'a_n\hat{\eta}_{1,n} + c'\hat{\eta}_{1,n}a_n^2 \biggl(w - \bar{c} - \bar{c} wa_n + \frac{3}{c'} \biggl).
		\end{aligned}
		\]
		Consider the function
		\[
		{f}(x) = c'\hat{\eta}_{1,n}x^2 \biggl(w - \bar{c} - \bar{c}wx + \frac{3}{c'} \biggl),
		\]
		which has two roots at $x_{1,2}=0$ and one root at $x_3=\frac{w - (\bar{c} - \frac{3}{c'})}{cw}$.
		Because  $\bar{c} - \frac{3}{c'} > 0$, we can  choose $0<w < \bar{c} - \frac{3}{c'}$ so that $x_3 < 0$.
		So ${f}(x) < 0$ when $x > 0$, leading to
		\[ c'\hat{\eta}_{1,n}a_n^2 \biggl(w- \bar{c} - \bar{c}wa_n + \frac{3}{c'} \biggl) < 0, \;\;\forall n
		\]
		because $a_n>0$. We have now proved $\E [|U_{1, n+1}|^2] \leq c' a_n \hat{\eta}_{1,n}$.
		
		In particular, under the setting of Remark \ref{remark_example_1}, we can verify that $(\frac{b_n}{a_n})|\beta_{1,n}|^2$ is a non-decreasing sequence of $n$, and  $a_n=\Theta(n^{-\frac{3}{4}})$. Then
		\begin{equation}
			\begin{aligned}
				\hat{\eta}_{1,n}=& c b_n ( 1 +  c_{1,n}^{8} + (\log b_n)^8 + \frac{|\beta_{1,n}|^2}{a_n} ) \exp{\{cc_{1,n}^6\}}\\
				\leq& c n^{\frac{1}{4}} (1 + (1 \vee \log \log n)^{\frac{4}{3}} + (0 \vee \log n)^8)(e \vee \log n)^c\\
				\leq& c n^{\frac{1}{4}} (1 \vee \log n)^{p} (1 \vee \log \log n)^{\frac{4}{3}},
			\end{aligned}
		\end{equation}
		where $c$ and $p$ are positive constants.
		The proof is now complete.
	\end{proof}

	\subsection{Analyzing \(\Bar{J}(\phi_1, \phi_2)\)}
	
	Recall that \(\Bar{J}(\phi_1, \phi_2)=J(0,x_0;\pi)\) with $\gamma=0$, where $\pi=\mathcal{N}(\cdot|\phi_1 x, \phi_2)$.
	
	\begin{lemma}
		\label{lemma_j_f_g_main}
		The value function can be expressed as
		\[\Bar{J}(\phi_1, \phi_2) = f(a(\phi_1)) + (\Sj D_j^\top \phi_2 D_j) g(a(\phi_1)),\]
		where $a(\phi_1)$ is defined by \eqref{eq_a_phi_1} and the functions \(f\) and \(g\) are defined as follows:
		\begin{equation}
			\label{eq_f_definition}
			f(a) =
			\begin{cases}
				\frac{x_0^2(-H - QT)}{2} & \text{if } a = 0, \\
				\frac{1}{2a}(Q - e^{aT}Q - He^{aT}a)x_0^2 & \text{if } a \neq 0,
			\end{cases}
		\end{equation}
		\begin{equation}
			\label{eq_g_definition}
			g(a) =
			\begin{cases}
				\frac{T(-2H - QT)}{4} & \text{if } a = 0, \\
				\frac{1}{2a^2}(QTa + Q + Ha - e^{aT}Q - He^{aT}a) & \text{if } a \neq 0.
			\end{cases}
		\end{equation}
	\end{lemma}
	
	\begin{proof}
		The value function of the policy $\mathcal{N}(\cdot|\phi_1 x, \phi_2)$ with $\gamma=0$ is (with a slight abuse of notation)
		\[
		J(t,x; \phi_1, \phi_2) =  \mathbb{E}\left[ -\frac{1}{2}\int_t^T Qx^{\bm\phi}(s)^2 \dd s - \frac{1}{2}Hx^{\bm\phi}(T)^2 \Big| x^{\bm\phi}(t)=x\right],
		\]
		where $\bm\phi=(\phi_1, \phi_2)$ and $\{x^{\bm\phi}(s):t \leq s \leq T\}$ is the corresponding exploratory state process.
		By the Feynman--Kac formula, $J(\cdot,\cdot; \phi_1, \phi_2)$  satisfies
		\begin{equation}
			\begin{aligned}
				\label{eq_hjb}
				\left\{
				\begin{array}{l}
					v_t+  \frac{1}{2} \int_{\mathcal{R}^l} \sum_{j=1}^m (C_jx + D_j^\top u)^2\mathcal{N}(u|\phi_1 x, \phi_2) \dd u v_{xx} \\
					\qquad + \left(Ax + B^\top \int_{\mathcal{R}^l}u\mathcal{N}(u|\phi_1 x, \phi_2)\dd u\right)v_x - \frac{1}{2}Qx^2 = 0,\\
					v(T,x) = -\frac{1}{2}Hx^2.
				\end{array}
				\right.
			\end{aligned}
		\end{equation}
		The solution to  the above PDE is
		\begin{equation}
			\label{eq_value_v}
			\begin{aligned}
				J(t,x; \phi_1, \phi_2)=&\frac{1}{2}\left[\frac{Q}{a(\phi_1)} - e^{a(\phi_1)(T-t)}(\frac{Q}{a(\phi_1)}+H)\right]x^2 - \frac{1}{2} (\Sj D_j^\top \phi_2 D_j)\\
				& \left[\frac{Q}{a(\phi_1)}t+\frac{e^{a(\phi_1)(T-t)}}{a(\phi_1)}(\frac{Q}{a(\phi_1)}+H) - (\frac{QT}{a(\phi_1)}+\frac{Q}{a(\phi_1)^2}+\frac{H}{a(\phi_1)})\right]
			\end{aligned}
		\end{equation}
		if $a(\phi_1)\neq0$, and
		\begin{equation}
			\label{eq_value_v2}
			\begin{aligned}
				J(t,x; \phi_1, \phi_2)=&\frac{1}{2}(Qt-QT-H)x^2 + \frac{1}{4} (\Sj D_j^\top \phi_2 D_j)(Qt^2-2QTt-2Ht)\\
				&-\frac{1}{4} (\Sj D_j^\top \phi_2 D_j)(QT^2+2HT)
			\end{aligned}
		\end{equation}
		if $a(\phi_1)=0$.
		
		The desired result follows immediately noting that $\Bar{J}(\phi_1, \phi_2)=J(0, x_0;\phi_1, \phi_2)$.
	\end{proof}

	
	\begin{lemma}
		\label{lemma_j_properties}
		Both the functions $f$ and $g$ defined respectively by \eqref{eq_f_definition} and \eqref{eq_g_definition} are continuously differentiable, monotonically non-increasing, and strictly negative everywhere.
	\end{lemma}

	\begin{proof}
		First of all, it is straightforward to check by L'H\^opital's rule that both $f$ and $g$ are continuous at $0$; hence they are continuous functions. Next,
		when $a\neq0$,
		\[
		\begin{aligned}
			f'(a) =& -\frac{Qx_0^2(1 + aTe^{aT}-e^{aT}) }{2a^2} - \frac{x_0^2}{2}H T e^{aT} \leq 0,
		\end{aligned}
		\]
		where the inequality is due to the facts that  $1+xe^x-e^x \geq 0$ $\forall x$ and that $Q, H \geq 0$. Moreover, again by L'H\^opital's rule we obtain
		\[ f'(0)=-\frac{T x_0^2 (2H + QT)}{4}=\lim_{a \to 0}f'(a),\]
		implying that $f$ is continuously differentiable at 0, and hence continuously differentiable everywhere and monotonically non-increasing. Similarly we can prove that $g$ is also continuously differentiable everywhere and monotonically non-increasing.

		Finally, it is evident that
		\[
		\begin{aligned}
			&\lim_{a \to -\infty} f(a)=\lim_{a \to -\infty} g(a) = 0,\;\;\;
			\lim_{a \to \infty} f(a) = \lim_{a \to \infty} g(a) = -\infty.
		\end{aligned}
		\]
		Thus, in view of the proved monotonicity,  we have $f(a)<0$ and $g(a)<0$ for any $a$.
	\end{proof}
	
	
		
	
\end{document}